%% file: main.tex
\documentclass{article}

\input{macros}

\title{Preference-Based Batch and Sequential Teaching:\\Towards a Unified View of Models}

\author{
  \textbf{Farnam Mansouri}\textsuperscript{\textdagger} \ \ \ \ 
  \textbf{Yuxin Chen}\textsuperscript{\textdaggerdbl} \ \ \ \ 
  \textbf{Ara Vartanian}\textsuperscript{$\star$} \ \ \  \ 
  \textbf{Xiaojin Zhu}\textsuperscript{$\star$} \ \ \ \ 
  \textbf{Adish Singla}\textsuperscript{\textdagger}\\
  \\
  \textsuperscript{\textdagger}Max Planck Institute for Software Systems (MPI-SWS), \texttt{\{mfarnam, adishs\}@mpi-sws.org}, \\  
  \textsuperscript{\textdaggerdbl}University of Chicago, \texttt{chenyuxin@uchicago.edu}, \\
  \textsuperscript{$\star$}University of Wisconsin-Madison, \texttt{\{aravart, jerryzhu\}@cs.wisc.edu} \\
}

\begin{document}
\maketitle

\newtoggle{longversion}
\settoggle{longversion}{true}
\input{0_abstract}
\input{1_introduction}
\input{2_model}
\input{3_tdsigma}
\input{4_batch-family}
\input{5.1_seq-family}

\input{5.2_batch-vs-seq}
\input{5.3_linear-vcd}
\input{6_conclusion}
\input{6.1_acknowledgments}

\bibliography{references}

\iftoggle{longversion}{
\clearpage
\onecolumn
\appendix 
{\allowdisplaybreaks
\input{7.1_appendix_batch-models}
\clearpage
\input{7.2_appendix_seq-models_warmuth}

\clearpage
\input{7.3_appendix_seq-models_globalVS-local}
\clearpage
\input{7.4.1_appendix_seq-models_linearVCD_first-lemma_non-empty}
\input{7.4.3_appendix_seq-models_linearVCD_theorem}
}
}
{}
\end{document}

%% file: macros.tex
     \usepackage[final, nonatbib]{neurips_2019}

\usepackage[utf8]{inputenc} 
\usepackage[T1]{fontenc}    
\usepackage{url}            
\usepackage{booktabs}       
\usepackage{amsfonts}       
\usepackage{nicefrac}       
\usepackage{slashbox,multirow}
\usepackage{array,booktabs}
\newcolumntype{x}[1]{>{\centering\arraybackslash}p{#1}}

\usepackage{microtype}
\usepackage{graphicx}
\usepackage{amsmath,amssymb}
\usepackage{mathabx} 
\usepackage{subcaption}
\usepackage{booktabs} 
\usepackage{xspace}
\usepackage{enumerate}
\usepackage{mathrsfs}
\usepackage{etoolbox}

\usepackage{xcolor}

\usepackage{algorithm}
\usepackage{algcompatible}
\usepackage[noend]{algpseudocode}

\newboolean{showcomments}
\setboolean{showcomments}{true}

\newcommand{\BlackBox}{\rule{1.5ex}{1.5ex}}  
\newenvironment{proof}{\par\noindent{\bf Proof\ }}{\hfill\BlackBox\\[2mm]}
\newtheorem{example}{Example}
\newtheorem{theorem}{Theorem}
\newtheorem{lemma}[theorem]{Lemma}

\newtheorem{defn}{Theorem}
\newtheorem{definition}[defn]{Definition}

\newcommand{\yuxin}[1]{\ifthenelse{\boolean{showcomments}}{\textcolor{blue}{YC: #1}}{}}
\newcommand{\adish}[1]{\ifthenelse{\boolean{showcomments}}{\textcolor{red}{AS: #1}}{}}

\newcommand{\commentout}[1]{}

\newcommand{\denselist}{
\itemsep -4pt\topsep-8pt\partopsep-8pt\itemindent 0pt
}

\newcommand{\Hypotheses}{\mathcal{H}}
\newcommand{\hypotheses}{H}

\newcommand{\hstar}{\hypothesis^\star}
\newcommand{\hinit}{{\hypothesis_0}}

\newcommand{\Examples}{\mathcal{Z}}
\newcommand{\Instances}{\mathcal{X}}
\newcommand{\Clabels}{\mathcal{Y}}

\newcommand{\examples}{Z}

\newcommand{\instances}{X}

\renewcommand{\example}{{z}}
\newcommand{\instance}{{x}}
\newcommand{\clabel}{{y}}

\newcommand{\TD}{\textsf{TD}\xspace}
\newcommand{\PBTD}{\textsf{PBTD}\xspace}
\newcommand{\NCTD}{\textsf{NCTD}\xspace}
\newcommand{\RTD}{\textsf{RTD}\xspace}

\newcommand{\VCD}{\textsf{VCD}\xspace}

\newcommand{\paren} [1] {\ensuremath{ \left( {#1} \right) }}

\def \argmin {\mathop{\rm arg\,min}}

\newcommand{\reals}{\ensuremath{\mathbb{R}}}

\newcommand{\hypothesis}[0]{\ensuremath{h}}

\newcommand{\bigO}[1]{\ensuremath{O\paren{#1}}}

\newcommand{\bigOmega}[1]{\ensuremath{\Omega\paren{#1}}}

\numberwithin{equation}{section}

\newcommand{\defref}[1]{Definition~\ref{#1}}
\newcommand{\tableref}[1]{Table~\ref{#1}}
\newcommand{\figref}[1]{Figure~\ref{#1}}

\newcommand{\secref}[1]{\S\ref{#1}}

\newcommand{\thmref}[1]{Theorem~\ref{#1}}

\newcommand{\lemref}[1]{Lemma~\ref{#1}}
\renewcommand{\algref}[1]{Algorithm~\ref{#1}}

\newcommand{\lnref}[1]{Line~\ref{#1}}

\usepackage{subcaption}
\usepackage{wrapfig}

\newcommand{\pref}{\sigma}

\newcommand{\CF}{\textsf{CF}}

\newcommand{\Val}{D}
\newcommand{\Candidate}{\mathbf{C}}

\newcommand{\Teacher}{\mathrm{T}}

\newcommand{\TeachingSeq}{\mathcal{S}}

\newcommand{\Sigmatd}[1]{\Sigma_{#1}{\text -}\TD}

\newcommand{\SigmaConst}{\Sigma_{\textsf{const}}}
\newcommand{\SigmaGlobal}{\Sigma_{\textsf{global}}}
\newcommand{\SigmaGvs}{\Sigma_{\textsf{gvs}}}
\newcommand{\SigmaLocal}{\Sigma_{\textsf{local}}}
\newcommand{\SigmaLvs}{\Sigma_{\textsf{lvs}}}

\newcommand{\gvs}{\textsf{gvs}}
\newcommand{\glbl}{\textsf{global}}
\newcommand{\lvs}{\textsf{lvs}}
\newcommand{\lcl}{\textsf{local}}
\newcommand{\SigmaConstTD}{\Sigmatd{\textsf{const}}}
\newcommand{\SigmaGlobalTD}{\Sigmatd{\textsf{global}}}
\newcommand{\SigmaGvsTD}{\Sigmatd{\textsf{gvs}}}
\newcommand{\SigmaLocalTD}{\Sigmatd{\textsf{local}}}
\newcommand{\SigmaLvsTD}{\Sigmatd{\textsf{lvs}}}

\usepackage{mathtools}
\DeclarePairedDelimiter{\ceil}{\lceil}{\rceil}

\newcommand{\hclass}[1]{H^{#1}}
\newcommand{\compactDSet}[1]{\Psi_{#1}}
\newcommand{\indexof}[1]{I{(#1)}}
\newcommand{\hnext}{\hypothesis_{\text{next}}}

\usepackage{tikz}
\usetikzlibrary{matrix,trees,backgrounds,fit,calc}

\def\constCircle{(4,0) circle (0.1cm)}

\def\gvsEllipse{(5,0) ellipse (2.2cm and 1.5cm)}
\def\localEllipse{(3,0) ellipse (2.2cm and 1.5cm)}
\def\lvsEllipse{(4,0) ellipse (4cm and 1.8cm)}

\newcommand{\distfun}{g}



%% file: 0_abstract.tex
\begin{abstract}
Algorithmic machine teaching studies the interaction between a teacher and a learner where the teacher selects labeled examples aiming at teaching a target hypothesis. In a quest to lower teaching complexity and to achieve more natural teacher-learner interactions, several teaching models and complexity measures have been proposed for both the batch settings (e.g., worst-case, recursive, preference-based, and non-clashing models) as well as the sequential settings (e.g., local preference-based model). To better understand the connections between these different batch and sequential models, we develop a novel framework which captures the teaching process via preference functions $\Sigma$. In our framework, each function $\sigma \in \Sigma$ induces a teacher-learner pair with teaching complexity as $\TD(\sigma)$. We show that the above-mentioned teaching models are equivalent to specific types/families of preference functions in our framework. This equivalence, in turn, allows us to study the differences between two important teaching models, namely $\sigma$ functions inducing the strongest batch (i.e., non-clashing) model and $\sigma$ functions inducing a weak sequential (i.e., local preference-based) model.  Finally, we identify preference functions inducing a novel family of sequential models with teaching complexity linear in the VC dimension of the hypothesis class: this is in contrast to the best known complexity result for the batch models which is quadratic in the VC dimension.
\end{abstract}

%% file: 1_introduction.tex
\vspace{-4mm}
\section{Introduction}\label{sec:intro}
\vspace{-2mm}
Algorithmic machine teaching studies the interaction between a teacher and a learner where the teacher’s goal is to find an optimal training sequence to steer the learner towards a target hypothesis \cite{goldman1995complexity,zilles2011models,zhu2013machine,singla2014near,zhu2015machine,DBLP:journals/corr/ZhuSingla18}. An important quantity of interest is the teaching dimension (TD) of the hypothesis class, representing the worst-case number of examples needed to teach any hypothesis in a given class. Given that the teaching complexity depends on what assumptions are made about teacher-learner interactions, different teaching models lead to different notions of teaching dimension. In the past two decades, several such teaching models have been proposed, primarily driven by the motivation to lower teaching complexity and to find models for which the teaching complexity has better connections with learning complexity measured by Vapnik–Chervonenkis dimension (VCD)~\cite{vapnik1971uniform} of the class.

Most of the well-studied teaching models are for the batch setting (e.g., worst-case \cite{goldman1995complexity,kuhlmann1999teaching}, recursive \cite{zilles2008teaching,zilles2011models,doliwa2014recursive}, preference-based \cite{gao2017preference}, and non-clashing~\cite{pmlr-v98-kirkpatrick19a} models). In these batch models, the teacher first provides a set of examples to the learner and then the learner outputs a hypothesis. In a quest to achieve more natural teacher-learner interactions and enable richer applications, various different models have been proposed for the sequential setting (e.g., local preference-based model for version space learners~\cite{chen2018understanding}, models for gradient learners~\cite{liu2017iterative,DBLP:conf/icml/LiuDLLRS18,DBLP:conf/ijcai/KamalarubanDCS19}, models inspired by control theory~\cite{zhu2018optimal,DBLP:conf/aistats/LessardZ019}, models for sequential tasks~\cite{cakmak2012algorithmic,haug2018teaching,tschiatschek2019learner}, and models for human-centered applications that require adaptivity~\cite{singla2013actively,hunziker2018teaching}).


%

In this paper, we seek to gain a deeper understanding of how different teaching models relate to each other.
To this end,  we develop a novel teaching framework which captures the teaching process via preference functions $\Sigma$. Here, a preference function $\sigma \in \Sigma$ models how a learner navigates in the version space as it receives teaching examples (see \secref{sec:model} for formal definition); in turn, each function $\sigma$ induces a teacher-learner pair with teaching dimension  $\TD(\sigma)$ (see \secref{sec:complexity}). We highlight some of the key results below:
\begin{itemize}
    \item We show that the well-studied teaching models in batch setting corresponds to specific families of $\sigma$ functions in our framework (see \secref{sec:non-seq-family-pref} and \tableref{tab:overview}).
    \item We study the differences in the family of $\sigma$ functions inducing the strongest batch model~\cite{pmlr-v98-kirkpatrick19a} and functions inducing a weak sequential model~\cite{chen2018understanding} (\secref{sec:connection}) (also, see the relationship between $\SigmaGvs$ and $\SigmaLocal$ in \figref{fig:venndiagram_seq}).
    \item  We identify preference functions inducing a novel family of sequential models with teaching complexity linear in the VCD of the hypothesis class. We provide a constructive procedure to find such $\sigma$ functions with low teaching complexity (\secref{sec:lvs_linvcd}).
\end{itemize}
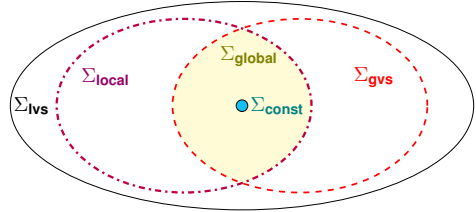
\begin{wrapfigure}{r}{0.42\textwidth}
\centering
\scalebox{.77}{
\begin{tikzpicture}
  \begin{scope}
    \clip \gvsEllipse;
    \fill[yellow!20] \localEllipse;
  \end{scope}
  \draw [fill=white!30!cyan] \constCircle node[text={rgb:green,3;blue,3}, right=.02cm] {\textbf{$\SigmaConst$}};
  \draw (3,0) node[text={rgb:red,7;green,7;blue,0},above=0.85cm,right=0.5cm] {\textbf{$\SigmaGlobal$}};
  \draw [red,line width=0.3mm,dashed] \gvsEllipse node[text={rgb:red,5;green,0;blue,0},above=.5cm,right=.8cm] {\textbf{$\SigmaGvs$}};
  \draw [purple,line width=0.4mm,dashdotted]\localEllipse node[text={rgb:red,1;purple,5;blue,2},above=.5cm,left=0.8cm] {\textbf{$\SigmaLocal$}};
  \draw \lvsEllipse node[text=black,left=3.2cm] {\textbf{$\SigmaLvs$}};
\end{tikzpicture}}
\caption{Venn diagram for different families of preference functions.} \label{fig:venndiagram_seq}
\end{wrapfigure}

Our key findings are highlighted in \figref{fig:venndiagram_seq} and \tableref{tab:overview}. Here, \figref{fig:venndiagram_seq} illustrates the relationship between different families of preference functions that we introduce, and \tableref{tab:overview} summarizes the key complexity results we obtain for different families. Our unified view of the existing teaching models in turn opens up several intriguing new directions such as (i) using our constructive procedures to design preference functions for addressing open questions of whether \RTD / \NCTD is linear in \VCD, and (ii) understanding the notion of collusion-free teaching in sequential models. We discuss these directions further in \secref{sec:discussion}.
\begin{table}[!htbp]
\centering
\scalebox{0.9}{
\begingroup
\renewcommand{\arraystretch}{1.25}
\begin{tabular}{c|c|c|c|c|c}
    Families & $\SigmaConst$ & $\SigmaGlobal$ & $\SigmaGvs$ & $\SigmaLocal$ & $\SigmaLvs$\\ 
    \midrule
    Reduction &  $\TD$ & $\RTD$ / $\PBTD$ & $\NCTD$& 
    Local-\PBTD & -- \\
    Complexity Results & -- & $O(\VCD^2)$ & $O(\VCD^2)$ & $O(\VCD^2)$ & $O(\VCD)$\\
    \hline
    &\cite{goldman1995complexity}&\cite{zilles2011models,gao2017preference,DBLP:journals/corr/HuWLW17}&\cite{pmlr-v98-kirkpatrick19a}&\cite{chen2018understanding}&
\end{tabular}
\endgroup
}
\caption{Overview of our main results -- reduction to existing models and teaching complexity.}\label{tab:overview} 
\end{table}




%% file: 2_model.tex
\vspace{-8mm}
\section{The Teaching Model}\label{sec:model}

\paragraph{The teaching domain.}
Let $\Instances$, $\Clabels$ be a ground set of unlabeled instances and the set of labels. 
Let $\Hypotheses$ be a finite class of hypotheses; each element $\hypothesis\in \Hypotheses$ is a function $\hypothesis: \Instances \rightarrow \Clabels$. Here, we only consider boolean functions and hence $\Clabels = \{0,1\}$. 
In our model, $\Instances$, $\Hypotheses$, and $\Clabels$ are known to both the teacher and the learner. There is a target hypothesis $\hstar\in \Hypotheses$ that is known to the teacher, but not the learner.  Let $\Examples \subseteq \Instances \times \Clabels$ be the ground set of labeled examples. Each element $\example = (\instance_\example,\clabel_\example) \in \Examples$ represents a labeled example where the label is given by the target hypothesis $\hstar$, i.e., $\clabel_\example = \hstar(\instance_\example)$.
%
%
For any $\examples \subseteq \Examples$, the \emph{version space} induced by $\examples$ is the subset of hypotheses $\Hypotheses(\examples) \subseteq \Hypotheses$ that are consistent with the labels of all the examples, i.e., $\Hypotheses(\examples):= \{\hypothesis\in \Hypotheses \mid \forall \example = (\instance_\example, \clabel_\example) \in \examples, h(\instance_\example) = \clabel_\example\}$.

\paragraph{Learner's preference function.}
\looseness -1 We consider a generic model of the learner that captures our assumptions about how the learner adapts her hypothesis based on the labeled examples received from the teacher. A key ingredient of this model is the learner's \emph{preference function} over the hypotheses. 
The learner, based on the information encoded in the inputs of preference function---which include the current hypothesis and the current version space---will choose one hypothesis in $\Hypotheses$. Our model of the learner strictly generalizes the local preference-based model considered in \cite{chen2018understanding}, where the learner's preference was only encoded by her current hypothesis.
Formally, we consider preference functions of the form $\sigma: \Hypotheses \times 2^\Hypotheses \times \Hypotheses \rightarrow \reals$.
For any two hypotheses $\hypothesis', \hypothesis''$, we say that the learner prefers $\hypothesis'$ to $\hypothesis''$ based on the current hypothesis $\hypothesis$ and version space $\hypotheses \subseteq \Hypotheses$, iff $\sigma(\hypothesis' ; \hypotheses, \hypothesis) < \sigma(\hypothesis'';\hypotheses, \hypothesis)$. If $\sigma(\hypothesis';\hypotheses, \hypothesis) = \sigma(\hypothesis'';\hypotheses, \hypothesis)$, then the learner could pick either one of these two.

\paragraph{Interaction protocol and teaching objective.}
The teacher's goal is to steer the learner towards the target hypothesis $\hstar$ by providing a sequence of labeled examples. The learner starts with an initial hypothesis $\hinit \in \Hypotheses$ before receiving any labeled examples from the teacher. 
At time step $t$, the teacher selects a labeled example $\example_t \in \Examples$, and the learner makes a transition from the current hypothesis to the next hypothesis.
Let us denote the labeled examples received by the learner up to (and including) time step $t$ via $\examples_{t}$. Further, we denote the learner's version space at time step $t$ as $\hypotheses_t=\Hypotheses(\examples_{t})$, and the learner's hypothesis before receiving $\example_t$ as $\hypothesis_{t-1}$. The learner picks the next hypothesis based on the current hypothesis $\hypothesis_{t-1}$, version space $\hypotheses_{t}$, and preference function $\sigma$:
  \begin{align}
    \hypothesis_{t} \in \argmin_{\hypothesis'\in \hypotheses_{t}} \sigma(\hypothesis'; \hypotheses_t, \hypothesis_{t-1}).
    \label{eq.learners-jump}
  \end{align}

Upon updating the hypothesis $\hypothesis_t$, the learner sends $\hypothesis_t$ as feedback to the teacher. Teaching finishes here if the learner's updated hypothesis $\hypothesis_t$ equals $\hstar$. 
%
We summarize the interaction in Protocol~\ref{alg:interaction}.\footnote{It is important to note that in our teaching model, the teacher and the learner use the same preference function. This assumption of shared knowledge of the preference function is also considered in existing teaching models for both the batch settings (e.g., as in \cite{zilles2011models,gao2017preference}) and the sequential settings \cite{chen2018understanding}).}



\makeatletter
\renewcommand{\ALG@name}{Protocol}
\makeatother

\begin{algorithm}[h!]
  \caption{Interaction protocol between the teacher and the learner}\label{alg:interaction}
    \begin{algorithmic}[1]
        \State learner's initial version space is $\hypotheses_{0} = \Hypotheses$ and learner starts from an initial hypothesis $\hinit \in \Hypotheses$
        \For{$t=1, 2, 3, \ldots$}
            \State learner receives $\example_t = (\instance_t, \clabel_t)$; updates  $\hypotheses_t=\hypotheses_{t-1} \cap \Hypotheses(\{\example_t\})$; picks $\hypothesis_t$ per Eq.~\eqref{eq.learners-jump};
            \State teacher receives $\hypothesis_t$ as feedback from the learner;
            \State \textbf{if \ } $\hypothesis_t = \hstar$ \textbf{then \ } teaching process terminates
        \EndFor
	\end{algorithmic}
\end{algorithm}



%% file: 3_tdsigma.tex

\vspace{-4mm}
\section{The Complexity of Teaching}\label{sec:complexity}
\vspace{-1mm}
\subsection{Teaching Dimension for a Fixed Preference Function}
Our objective is to design teaching algorithms that can steer the learner towards the target hypothesis in a minimal number of time steps. We study the \emph{worst-case} number of steps needed, as is common when measuring information complexity of teaching~\cite{goldman1995complexity,zilles2011models,gao2017preference,zhu2018optimal}. 
Fix the ground set of instances $\Instances$ and the learner's preference $\sigma$. For any version space $\hypotheses\subseteq\Hypotheses$, the worst-case optimal cost for steering the learner from $\hypothesis$ to $\hstar$ is characterized by
\begin{align*}
  \Val_{\pref}(\hypotheses, \hypothesis, \hstar) =
  \begin{cases}
    1, & 
    \exists z, \text{~s.t.~}\Candidate_{\sigma}(H,h,z) = \{h^*\}\\
    1 + \min_{\example} \max_{\hypothesis'' \in \Candidate_\sigma(\hypotheses, \hypothesis, \example) 
    } \Val_{\pref}(\hypotheses \cap \Hypotheses(\{\example\}), \hypothesis'', \hstar) ,  &\text{otherwise}
  \end{cases}
\end{align*}
where 
    $\Candidate_\sigma(\hypotheses, \hypothesis, \example) =\argmin_{\hypothesis' \in \hypotheses \cap \Hypotheses(\{\example\})} \sigma(\hypothesis'; \hypotheses \cap \Hypotheses(\{\example\}), \hypothesis)$
denotes the set of candidate hypotheses most preferred by the learner. 
Note that our definition of teaching dimension is similar in spirit to the local preference-based teaching complexity defined by \cite{chen2018understanding}. We shall see in the next section, this complexity measure in fact reduces to existing notions of teaching complexity for specific families of preference functions.

Given a preference function $\pref$ and the learner's initial hypothesis $\hinit$, the teaching dimension w.r.t. $\pref$ is defined as the worst-case optimal cost for teaching any target $\hstar$:
\begin{align}
    \TD_{\Instances,\Hypotheses,\hinit}(\pref) = \max_{\hstar} \Val_{\pref}(\Hypotheses, \hinit, \hstar).
\end{align}
\subsection{Teaching Dimension for a Family of Preference Functions}

In this paper, we will investigate several families of preference functions (as illustrated in \figref{fig:venndiagram_seq}).
For a family of preference functions $\Sigma$, 
we define the teaching dimension w.r.t the family $\Sigma$ as the teaching dimension w.r.t. the \emph{best} $\pref$ in that family:
\begin{align}\label{eq:sigmatd}
        \Sigma_{}{\text -}\TD_{\Instances,\Hypotheses,\hinit} =  \min_{\sigma\in \Sigma} \TD_{\Instances,\Hypotheses,\hinit}(\pref).
\end{align}

\subsection{Collusion-free Preference Functions}
An important consideration when designing teaching models is to ensure that the teacher and the learner are ``collusion-free'', i.e., they are not allowed to collude or use some ``coding-trick'' to achieve arbitrarily low teaching complexity. 
%
A well-accepted notion of collusion-freeness in the batch setting is one proposed by \cite{goldman1996teaching} (also see \cite{angluin1997teachers,ott1999avoiding,pmlr-v98-kirkpatrick19a}). Intuitively, it captures the idea that a learner conjecturing hypothesis $\hypothesis$ will not change its mind when given additional information consistent with $\hypothesis$. 
In comparison to batch models, the notion of collusion-free teaching in the sequential models is not well understood. 
%
We introduce a novel notion of  collusion-freeness for the sequential setting, which captures the following idea: if $\hypothesis$ is the only hypothesis in the most preferred set defined by $\sigma$, then the learner 
    will always stay at $\hypothesis$ as long as additional information received by the learner is consistent with $\hypothesis$. We formalize this notion in the definition below. Note that for $\sigma$ functions corresponding to batch models (see \secref{sec:non-seq-family-pref}), \defref{def:seq-col-free} reduces to the collusion-free definition of \cite{goldman1996teaching}.

\begin{definition}[Collusion-free preference]\label{def:seq-col-free}
%
Consider a time $t$ where the learner’s current hypothesis is $\hypothesis_{t-1}$ and version space is $\hypotheses_{t}$ (see Protocol \ref{alg:interaction}). Further assume that the learner’s preferred hypothesis for time $t$ is uniquely given by $\argmin_{\hypothesis' \in \hypotheses_t} \sigma(\hypothesis'; \hypotheses_t, \hypothesis_{t-1})=\{\hat{\hypothesis}\}$. 
Let $S$ be additional examples provided by an adversary from time $t$ onwards. We call a preference function collusion-free, if for any $S$ consistent with $\hat{\hypothesis}$, it holds that $\argmin_{\hypothesis' \in \hypotheses_t \cap \Hypotheses(S)} \sigma(\hypothesis'; \hypotheses_t \cap \Hypotheses(S), \hat{\hypothesis}) = \{\hat{\hypothesis}\}$.

%
%

\end{definition}

 
In this paper, we study preference functions that are collusion-free. In particular, we use $\Sigma_{\CF}$ to denote the set of preference functions that induce collusion-free teaching: 
\begin{align*}
    \Sigma_{\CF} = \{\sigma\mid \sigma \text{ is collusion-free}\}. 
\end{align*}

%% file: 4_batch-family.tex
\vspace{-4mm}
\section{Preference-based Batch Models} \label{sec:non-seq-family-pref}
\vspace{-1mm}

\subsection{Families of Preference Functions}
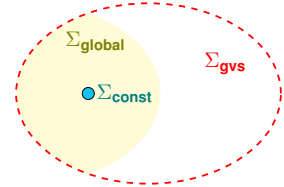
\begin{wrapfigure}{r}{0.26\textwidth}
\centering
\scalebox{.8}{
\begin{tikzpicture}
  \begin{scope}
    \clip \gvsEllipse;
    \fill[yellow!20] \localEllipse;
  \end{scope}
  \draw [fill=white!30!cyan] \constCircle node[text={rgb:green,3;blue,3}, right=.02cm] {\textbf{$\SigmaConst$}};
  \draw (3,0) node[text={rgb:red,10;green,10;blue,0},above=0.85cm,right=0.5cm] {\textbf{$\SigmaGlobal$}};
  \draw [red,line width=0.3mm,dashed] \gvsEllipse
  node[text={rgb:red,5;green,0;blue,0},above=.5cm,right=.8cm] {\textbf{$\SigmaGvs$}};
\end{tikzpicture}}
\caption{Batch models.} \label{fig:venndiagram_batch}
\end{wrapfigure}
We consider three families of preference functions which do not depend on the learner's current hypothesis. 
The first one is the family of uniform preference functions, denoted by $\SigmaConst$, which corresponds to constant preference functions:
\begin{align*}
    \SigmaConst =  \{\sigma \in  \Sigma_{\CF} \mid \exists c \in \reals, \text{~s.t.~} \forall \hypothesis', \hypotheses, \hypothesis, \sigma(\hypothesis'; \hypotheses, \hypothesis) = c\}
\end{align*}
The second family, denoted by $\SigmaGlobal$, corresponds to the preference functions that do not depend on the learner's current hypothesis and version space. In other words, the preference functions capture some \emph{global} preference ordering of the hypotheses:
\begin{align*}
    \SigmaGlobal =  \{\sigma \in \Sigma_{\CF}\mid \exists~\distfun: \Hypotheses \rightarrow \reals, \text{~s.t.~}\forall \hypothesis', \hypotheses, \hypothesis,~\sigma(\hypothesis'; \hypotheses, \hypothesis) = \distfun(\hypothesis')\}
\end{align*}
The third family, denoted by $\SigmaGvs$, corresponds to the preference functions that depend on the learner's version space, but do not depend on the learner's current hypothesis:
\begin{align*}
    \SigmaGvs =  \{\sigma \in \Sigma_{\CF} \mid \exists~\distfun: \Hypotheses\times 2^\Hypotheses \rightarrow \reals, \text{~s.t.~} \forall \hypothesis', \hypotheses, \hypothesis, \sigma(\hypothesis'; \hypotheses, \hypothesis) = \distfun(\hypothesis', \hypotheses)\}
\end{align*}
\figref{fig:venndiagram_batch} illustrates the relationship between these preference families. 
\vspace{-1mm}
\subsection{Complexity Results}
We first provide several definitions, including the formal definition of VC dimension as well as several existing notions of teaching dimension.

\begin{definition}[Vapnik–Chervonenkis dimension \cite{vapnik1971uniform}]
The VC dimension for $H \subseteq \Hypotheses$ w.r.t. a fixed set of unlabeled instances $X \subseteq \Instances$, denoted by $\VCD(H,X)$, is the cardinality of the largest set of points $X' \subseteq X$ that are ``shattered''.\footnote{In the classical definition of $\VCD$, only the first argument $H$ is present; the second argument $X$ is omitted and is by default the ground set of unlabeled instances $\Instances$.} Formally, let $\hypotheses_{|\instances} = \{(\hypothesis(\instance_1), ..., \hypothesis(\instance_n)) \ | \ \forall \hypothesis \in \hypotheses\}$ denote all possible patterns of $\hypotheses$ on $\instances$. Then $\VCD(H, X)=\max{|X'|}, \text{~s.t.~} X' \subseteq X \text{~and~} |H_{|X'}| = 2^{|X'|}$.
\end{definition}

\begin{definition}[Teaching dimension \cite{goldman1995complexity}] For any hypothesis $\hypothesis\in \Hypotheses$, we call a set of instances $\Teacher(\hypothesis) \subseteq \Instances$ a teaching set for $\hypothesis$, if it can uniquely identify $\hypothesis \in \Hypotheses$. 
The teaching dimension for $\Hypotheses$, denoted by $\TD(\Hypotheses)$, is the maximum size of the minimum teaching set for any $\hypothesis\in \Hypotheses$: $\TD(\Hypotheses) = \max_{\hypothesis\in\Hypotheses} \min |\Teacher(\hypothesis)|$.
\end{definition}

As noted by \cite{zilles2008teaching}, the teaching dimension of \cite{goldman1995complexity} does not always capture the intuitive idea of cooperation between teacher and learner. The authors then introduced a model of cooperative teaching that resulted in the complexity notion of recursive teaching dimension, as defined below.
\begin{definition}[Recursive teaching dimension 
\cite{zilles2008teaching,zilles2011models}]
The recursive teaching dimension (\RTD) of $\Hypotheses$, denoted by $\RTD(\Hypotheses)$, is the
smallest number $k$, such that one can find an ordered sequence of hypotheses in $\Hypotheses$, denoted by $(\hypothesis_1, \dots, \hypothesis_i, \dots, \hypothesis_{|\Hypotheses|})$, where every hypothesis $\hypothesis_i$ has a teaching set of size no more than $k$ to be distinguished from the hypotheses in the remaining sequence.
\end{definition}
In this paper we consider finite hypothesis classes. Under this setting, \RTD is equivalent to preference-based teaching dimension (\PBTD) \cite{gao2017preference}.
%
%
%

In a recent work of  \cite{pmlr-v98-kirkpatrick19a}, a new notion of teaching complexity, called non-clashing teaching dimension or \NCTD, was introduced (see definition below). Importantly, \NCTD is the optimal teaching complexity among teaching models in the batch setting that satisfy the collusion-free property of \cite{goldman1996teaching}. 
\begin{definition}[Non-clashing teaching dimension \cite{pmlr-v98-kirkpatrick19a}]
Let $\Hypotheses$ be a hypothesis class and $\Teacher: \Hypotheses \rightarrow 2^\Instances$ be a ``teacher mapping''
on $\Hypotheses$, i.e., mapping a given hypothesis to a teaching set.\footnote{We refer the reader to the original paper \cite{pmlr-v98-kirkpatrick19a} for a more formal description of ``teacher mapping".}
We say that $\Teacher$ is non-clashing on $\Hypotheses$ iff there are no two distinct $\hypothesis,\hypothesis'\in\Hypotheses$ such that $\Teacher(\hypothesis)$ is consistent with $\hypothesis'$ and $\Teacher(\hypothesis')$ is consistent with $\hypothesis$.
The non-clashing Teaching Dimension of $\Hypotheses$, denoted by $\NCTD(\Hypotheses)$, is defined as $\NCTD(\Hypotheses) = \min_{\Teacher \text{~is non-clashing}} \{\max_{h\in \Hypotheses} |\Teacher(h)|\}$.
\end{definition}

We show in the following, that the teaching dimension $\Sigma{\text-}\TD$ in Eq.~\eqref{eq:sigmatd} unifies the above definitions of \TD's for batch models.

\begin{theorem}[Reduction to existing notions of TD's] \label{theorem:equivelence-results} Fix $\Instances,\Hypotheses,\hinit$. The teaching complexity for the three families reduces to the existing notions of teaching dimensions:
\begin{enumerate}\denselist
\item $\SigmaConstTD_{\Instances,\Hypotheses,\hinit} = \TD(\Hypotheses)$
\item $\SigmaGlobalTD_{\Instances,\Hypotheses,\hinit} = \RTD(\Hypotheses) = O(\VCD(\Hypotheses, \Instances)^2) $
\item $\SigmaGvsTD_{\Instances,\Hypotheses,\hinit} = \NCTD(\Hypotheses) = O(\VCD(\Hypotheses, \Instances)^2)$
\end{enumerate}
\end{theorem}
Our teaching model strictly generalizes the local-preference based model of \cite{chen2018understanding}, which reduces to  the ``worst-case'' model when $\sigma \in \SigmaConst$ (corresponding to $\TD$) \cite{goldman1995complexity} and the global ``preference-based'' model when $\pref \in \SigmaGlobal$. Hence we get $\SigmaConstTD_{\Instances,\Hypotheses,\hinit} = \TD(\Hypotheses)$ and $\SigmaGlobalTD_{\Instances,\Hypotheses,\hinit} = \RTD(\Hypotheses)$. 
To establish the equivalence between $\SigmaGvsTD_{\Instances,\Hypotheses,\hinit}$ and $\NCTD(\Hypotheses)$, it suffices to show that for any $\Instances,\Hypotheses,\hinit$, the following holds: (i) $\SigmaGvsTD_{\Instances,\Hypotheses,\hinit} \geq \NCTD(\Hypotheses)$, and (ii) $\SigmaGvsTD_{\Instances,\Hypotheses,\hinit} \leq \NCTD(\Hypotheses)$.
\iftoggle{longversion}
{The full proof is provided in Appendix~\ref{sec.appendix.batch-models.theorem-proof}.}
{The full proof is provided in Appendix~A.2 of the supplementary.}

In \tableref{tab:main:warmuth_class_sequence_pref}, we consider the well known Warmuth hypothesis class \cite{doliwa2014recursive} where $\SigmaConstTD=3$, $\SigmaGlobalTD=3$, and $\SigmaGvsTD=2$.
\tableref{tab:pref_const_global} and \tableref{tab:pref_gvs} show preference functions $\sigma \in \SigmaConst$, $\sigma \in \SigmaGlobal$, and $\sigma \in \SigmaGvs$ that achieve the minima in Eq.~\eqref{eq:sigmatd}. \tableref{tab:warmth_example_sequences} shows the teaching sequences achieving these teaching dimensions for these preference functions.
%
\iftoggle{longversion}
{In Appendix~\ref{sec.appendix.batch_models.example},}
{In Appendix~A.1,}
we provide another hypothesis class where $\SigmaConstTD=3$, $\SigmaGlobalTD=2$, and $\SigmaGvsTD=1$.

%% file: 5.1_seq-family.tex
\vspace{-2mm}
\section{Preference-based Sequential Models} \label{sec:seq-family-pref}
\vspace{-2mm}
\subsection{Families of Preference Functions}
\looseness -1 In this section, we investigate two families of preference functions that depend on the learner's current hypothesis $\hypothesis_{t-1}$.
The first one is the family of local preference-based functions \cite{chen2018understanding}, denoted by $\SigmaLocal$, which corresponds to preference functions that depend on the learner's current (local) hypothesis, but do not depend on the learner's version space:
\begin{align*}
    \SigmaLocal =  \{\sigma \in \Sigma_{\CF} \mid 
    \exists~\distfun: \Hypotheses\times \Hypotheses \rightarrow \reals, 
    \text{~s.t.~} \forall \hypothesis', \hypotheses, \hypothesis,  \sigma(\hypothesis'; \hypotheses, \hypothesis) = \distfun(\hypothesis', \hypothesis)\}
\end{align*}
The second family, denoted by $\SigmaLvs$, corresponds to the preference functions that depend on 
all three arguments of $\sigma(\hypothesis'; \hypotheses, \hypothesis)$. 
The dependence of $\sigma$ on the learner's current (local) hypothesis and the version space renders a powerful family of preference functions:
\begin{align*}
    \SigmaLvs =  \{\sigma \in \Sigma_{\CF} \mid \exists~\distfun: \Hypotheses \times 2^\Hypotheses \times \Hypotheses \rightarrow \reals, \text{~s.t.~} \forall \hypothesis', \hypotheses, \hypothesis,\sigma(\hypothesis'; \hypotheses, \hypothesis) = \distfun(\hypothesis',\hypotheses, \hypothesis)\}
\end{align*}

\begin{table}[t!]\label{tab:warmth_example_families}
\centering
    \begin{subtable}[t]{\textwidth}
        \centering
        \scalebox{0.85}{
        \begin{tabular}{l|lllll||l|l|l|l}
        \backslashbox{$\hypothesis$}{$\instance$} & $\instance_1$ & 
        $\instance_2$ & 
        $\instance_3$ & 
        $\instance_4$ & 
        $\instance_5$ & 
        $\TeachingSeq_{\textsf{const}}=\TeachingSeq_{\textsf{\glbl}}$ &
        $\TeachingSeq_{\gvs}$ &
        $\TeachingSeq_{\lcl}$ &
        $\TeachingSeq_{\lvs}$ 
        \\ \hline
        $\hypothesis_1$ 
        & 1 & 1 & 0  & 0 & 0 
        & $\paren{\instance_1, \instance_2, \instance_4}$ 
        & $\paren{\instance_1, \instance_2}$
        & $\paren{\instance_1}$ 
        & $\paren{\instance_1}$ 
        \\
        $\hypothesis_2$
        & 0 & 1 & 1 & 0 & 0
        & $\paren{\instance_2, \instance_3, \instance_5}$ 
        & $\paren{\instance_2, \instance_3}$
        & $\paren{\instance_3}$ 
        & $\paren{\instance_2}$ 
        \\
        $\hypothesis_3$
        & 0 & 0 & 1 & 1 & 0
        & $\paren{\instance_1, \instance_3, \instance_4}$ 
        & $\paren{\instance_3, \instance_4}$
        & $\paren{\instance_3, \instance_4}$  
        & $\paren{\instance_3}$ 
        \\
        $\hypothesis_4$
        & 0 & 0 & 0 & 1 & 1
        & $\paren{\instance_2, \instance_4, \instance_5}$ 
        & $\paren{\instance_4, \instance_5}$
        & $\paren{\instance_5, \instance_4}$  
        & $\paren{\instance_4}$ 
        \\
        $\hypothesis_5$
        & 1 & 0 & 0 & 0 & 1
        & $\paren{\instance_1, \instance_3, \instance_5}$ 
        & $\paren{\instance_1, \instance_5}$
        & $\paren{\instance_5}$ 
        & $\paren{\instance_5}$ 
        \\
        $\hypothesis_6$
        & 1 & 1 & 0 & 1 & 0
        & $\paren{\instance_1, \instance_2, \instance_4}$ 
        & $\paren{\instance_2, \instance_4}$
        & $\paren{\instance_4}$ 
        & $\paren{\instance_3}$ 
        \\
        $\hypothesis_7$
        & 0 & 1 & 1 & 0 & 1
        & $\paren{\instance_2, \instance_3, \instance_5}$ 
        & $\paren{\instance_3, \instance_5}$
        & $\paren{\instance_3, \instance_5}$  
        & $\paren{\instance_4}$ 
        \\
        $\hypothesis_8$
        & 1 & 0 & 1 & 1 & 0
        & $\paren{\instance_1, \instance_3, \instance_4}$ 
        & $\paren{\instance_1, \instance_4}$
        & $\paren{\instance_4, \instance_3}$  
        & $\paren{\instance_5}$ 
        \\
        $\hypothesis_9$
        & 0 & 1 & 0 & 1 & 1
        & $\paren{\instance_2, \instance_4, \instance_5}$ 
        & $\paren{\instance_2, \instance_5}$
        & $\paren{\instance_4, \instance_5}$  
        & $\paren{\instance_1}$ 
        \\
        $\hypothesis_{10}$
        & 1 & 0 & 1 & 0 & 1 
        & $\paren{\instance_1, \instance_3, \instance_5}$ 
        & $\paren{\instance_1, \instance_3}$
        & $\paren{\instance_5, \instance_3}$ 
        & $\paren{\instance_2}$ 
        \end{tabular}}
        \caption{The Warmuth hypothesis class and the corresponding teaching sequences (denoted by $\TeachingSeq$).
        }\label{tab:warmth_example_sequences}
    \end{subtable}
    \begin{subtable}[t]{.3\textwidth}
        \centering
        \scalebox{0.9}{
        \begin{tabular}{c|c}
        $\hypothesis'$ & $\forall \hypothesis' \in \hypotheses$  \\\hline
        $\sigma_{\textsf{const}}(\hypothesis'; \cdot, \cdot)$ & \multirow{2}{*}{0}\\
        $\sigma_{\glbl}(\hypothesis'; \cdot, \cdot)$ \\
    \end{tabular}
    }
    \caption{
    $\sigma_{\textsf{const}}$ and $\sigma_{\glbl}$
    } \label{tab:pref_const_global}
    \end{subtable}
    \quad
    \begin{subtable}[t]{.6\textwidth}
        \centering
        \scalebox{0.9}{
        \begin{tabular}{c|p{.1cm}p{.1cm}p{.1cm}p{.1cm}p{.1cm}p{.1cm}p{.1cm}p{.1cm}p{.1cm}p{.1cm}}
        {$\hypothesis\backslash \hypothesis'$} & 
        $\hypothesis_1$ & $\hypothesis_2$ & $\hypothesis_3$ & $\hypothesis_4$ & $\hypothesis_5$ & $\hypothesis_6$ & $\hypothesis_7$ & $\hypothesis_8$ & $\hypothesis_9$ & $\hypothesis_{10}$ \\\hline
        $\sigma_{\lcl}(\hypothesis'; \cdot, \hypothesis=\hypothesis_1)$ & 
        0 & 2 & 4 & 4 & 2 & 1 & 3 & 3 & 3 & 3 \\
        {\footnotesize $\dots$}&  
    \end{tabular}
    }
    \caption{$\sigma_{\lcl}$ 
    representing the Hamming distance between $\hypothesis'$ and $\hypothesis$. } \label{tab:pref_local}
    \end{subtable}
        \begin{subtable}[t]{.3\textwidth}
    \scalebox{0.8}{
    \begin{tabular}{c|ccc}
    $\hypothesis'$ & $\hypothesis_1$ & $\hypothesis_2$ 
    &  $\dots$ \\ \hline 
         \multirow{2}{*}{$\hypotheses$} & 
         $\{\hypothesis_1, \hypothesis_6\}$ & 
         $\{\hypothesis_2, \hypothesis_7\}$ &
         \dots
         \\
        & 
         $\{\hypothesis_1\}$ & 
         $\{\hypothesis_2\}$ &
         $\dots$
         \\\hline
         $\sigma_\gvs$ & 
         0 & 0 & $\dots$\\
    \end{tabular}}
    \caption{$\sigma_\gvs(\hypothesis'; \hypotheses, \cdot)$
    }\label{tab:pref_gvs}
    \end{subtable}
    \qquad
     \begin{subtable}[t]{.5\textwidth}
        \centering
        \scalebox{0.7}{
        \begin{tabular}{c|cc|cc|c}
        $\hypothesis'$ & 
         \multicolumn{2}{c|}{$\hypothesis_1$} & 
         \multicolumn{2}{c|}{$\hypothesis_2$} & 
         $\dots$
         \\\hline
        \multirow{2}{*}{$\hypotheses$} & \multicolumn{2}{c|}{$\{\hypothesis_1\} \cup$} & \multicolumn{2}{c|}{$\{\hypothesis_2\} \cup$} &
         $\dots$
         \\
        &\multicolumn{2}{c|}{$\{\hypothesis_5, \hypothesis_6,\hypothesis_8,\hypothesis_{10}\}^*$}
        &\multicolumn{2}{c|}{$\{\hypothesis_1, \hypothesis_7,\hypothesis_6,\hypothesis_{9}\}^*$}
        & $\dots$
         \\\hline
         $\hypothesis$ & 
         \multicolumn{2}{c|}{$\hypothesis_1$} & 
         $\hypothesis_1$ & $\hypothesis_2$ &
         $\dots$
         \\\hline
        $\sigma_\lvs$ & 
        \multicolumn{2}{c|}{0} & 
        0 & 0 & $\dots$\\
        \end{tabular}
        }
    \caption{$\sigma_\lvs(\hypothesis'; \hypotheses, \hypothesis)$. 
    Here, $\{\cdot \}^*$ denotes all subsets.}\label{tab:pref_lvs}
    \end{subtable}
\vspace{-1mm}    
\caption{Teaching sequences with different preference functions for the Warmuth hypothesis class
\cite{doliwa2014recursive}.\protect\footnotemark~ \iftoggle{longversion}{Full preference functions are given in Appendix~\ref{sec.appendix.extension-of-table2}.
}
{Full preference functions are given in Appendix~B of the supplementary.
}
}\label{tab:main:warmuth_class_sequence_pref}
\vspace{-4mm}
\end{table}
\footnotetext{
The Warmuth hypothesis class is the smallest concept class for which \RTD exceeds \VCD.}

\looseness -1 \figref{fig:venndiagram_seq} illustrates the relationship between these preference families. 
As an example, in \tableref{tab:pref_local} and \tableref{tab:pref_lvs}, we provide the preference functions $\sigma_{\lcl}$ and $\sigma_{\lvs}$ for the Warmuth hypothesis class that achieve the minima in Eq.~\eqref{eq:sigmatd}. 







%% file: 5.2_batch-vs-seq.tex
\subsection{Comparing $\SigmaGvsTD$ and $\SigmaLocalTD$}
\label{sec:connection}


In the following, we show that substantial differences arise as we transition from $\sigma$ functions inducing the strongest batch (i.e., non-clashing) model to $\sigma$ functions inducing a weak sequential (i.e., local preference-based) model. We provide the full proof of \thmref{thm:local-eq-GVS}
\iftoggle{longversion}{
in Appendix~\ref{sec.appendix.seq-models_globalVS-local}.
}
{
in Appendix~C of the supplementary.
}

\begin{theorem}\label{thm:local-eq-GVS}
Neither of the families $\SigmaGvs$ and $\SigmaLocal$ dominates the other. Specifically,
\begin{enumerate}\denselist
    \item $\SigmaGvs\cap \SigmaLocal=\SigmaGlobal$
    \item There exist $\Hypotheses$, $\Instances$, where $\forall \hinit\in \Hypotheses, 
    \SigmaLocalTD_{\Instances,\Hypotheses,\hinit} > \SigmaGvsTD_{\Instances,\Hypotheses,\hinit}$    
    \item There exist $\Hypotheses$, $\Instances$, where $\forall \hinit\in \Hypotheses, 
    \SigmaLocalTD_{\Instances,\Hypotheses,\hinit} < \SigmaGvsTD_{\Instances,\Hypotheses,\hinit}$
\end{enumerate}
\end{theorem}

%% file: 5.3_linear-vcd.tex
\subsection{Complexity Results}\label{sec:lvs_linvcd}
We now connect the teaching complexity of the sequential models with the VC dimension.
\begin{theorem}\label{thm:main:seq-models_vs_VCD}
$\SigmaLocalTD_{\Instances,\Hypotheses,\hinit} = O(\VCD(\Hypotheses,\Instances)^2)$, and $\SigmaLvsTD_{\Instances,\Hypotheses,\hinit} = O(\VCD(\Hypotheses,\Instances))$.
\end{theorem}


To establish the proof, we first introduce an important definition (\defref{def:compactds}) and a key lemma (\lemref{lem:main:div_concepts}).

\begin{definition}[Compact-Distinguishable Set]\label{def:compactds}
Fix $\hypotheses\subseteq \Hypotheses$ and $\instances \subseteq \Instances$, where $\instances = \{\instance_1, ..., \instance_n\}$. Let $\hypotheses_{|\instances} = \{(\hypothesis(\instance_1), ..., \hypothesis(\instance_n)) \ | \ \forall \hypothesis \in \hypotheses\}$ denote all possible patterns of $\hypotheses$ on $\instances$. 
Then, we say that $\instances$ is \emph{compact-distinguishable} on $\hypotheses$, if $|\hypotheses_{|\instances}| = |\hypotheses|$ and $\forall \instances' \subset \instances,~ |\hypotheses_{|\instances'}| < |\hypotheses|$. We will use $\compactDSet{\hypotheses}$ to denote a compact-distinguishable set on $\hypotheses$. 
\end{definition}
In words, one can uniquely identify any hypothesis in $\hypotheses$ with a (sub)set of examples from $\compactDSet{\hypotheses}$ (also see the definition of distinguishing sets in \cite{doliwa2014recursive}). 
Our definition of compact-distinguishable set further implies that there are no ``redundant'' examples in $\compactDSet{\hypotheses}$. It can be shown that a compact-distinguishable set satisfies the following two properties: (i) it does not contain any pair of distinct instances $\instance,\instance'$ such that $(\forall \hypothesis \in \hypotheses: \hypothesis(\instance) = \hypothesis(\instance')) \textnormal{ or } (\forall \hypothesis \in \hypotheses: \hypothesis(\instance) \neq \hypothesis(\instance'))$; and (ii) it does not contain any instance $\instance$ such that $(\forall \hypothesis \in \hypotheses: \hypothesis(\instance) = 1) \textnormal{ or } (\forall \hypothesis \in \hypotheses: \hypothesis(\instance) = 0)$.

\begin{lemma}\label{lem:main:div_concepts}
Consider a subset $\hypotheses \subseteq \Hypotheses$ and any compact-distinguishable set $\compactDSet{\hypotheses} = \{\instance_1, ..., \instance_{|\compactDSet{\hypotheses}|}\}$. Fix any hypothesis $\hypothesis_\hypotheses \in \hypotheses$. 
Let $d = \VCD(\hypotheses, \compactDSet{\hypotheses})$ denote the VC dimension of $\hypotheses$ on $\compactDSet{\hypotheses}$. If $d \geq 1$, we can divide $\hypotheses$ into $m = |\compactDSet{\hypotheses}| + 1$ separate hypothesis classes 
$\{\hclass{1}, ..., \hclass{m}\}$, such that 
\begin{enumerate}[(i)]\denselist
    \item $\forall j \in [m]$, there exists a compact-distinguishable set $\compactDSet{\hclass{j}}$ s.t. $\VCD(\hclass{j}, \compactDSet{\hclass{j}}) \leq d-1$.
    \item $\forall j \in [m-1]$, $\hypotheses^j$ is not empty and $\hclass{j}_{|\{\instance_j\}} = \{( 1-  \hypothesis_\hypotheses(\instance_j))\}$. 
    \item $\hclass{m} = \{\hypothesis_\hypotheses\}$.
\end{enumerate}
\end{lemma}

\lemref{lem:main:div_concepts} suggests that for any 
$\Hypotheses, \Instances$, one can partition the hypothesis class $\Hypotheses$ into $m \leq |\Instances|+1$ subsets with lower VC dimension with respect to some compact-distinguishable set.\footnote{When $\VCD(\hypotheses,\compactDSet{\hclass{}})=0$, this implies $|\hypotheses|=1$.} The main idea of the lemma is similar to the reduction of a concept class w.r.t. some instance $x$ to lower \VCD as done in Theorem 9 of \cite{floyd1995sample}. The key distinction of \lemref{lem:main:div_concepts} is that we consider compact-distinguishable sets for this partitioning, which in turn ensures the uniqueness of the version spaces associated with these partitions (see proof of Theorem~\ref{thm:main:seq-models_vs_VCD}).
Another key novelty in our proof of Theorem~\ref{thm:main:seq-models_vs_VCD} is to recursively apply the reduction step from the lemma.


To prove the lemma, we provide a constructive procedure to partition the hypothesis class, and show that the resulting partitions have reduced VC dimensions on some compact-distinguishable set. We highlight the procedure for constructing the partitions in \algref{alg:lemma_linVCD} (\lnref{alg:ln:lemma_linvcd_start}-- \lnref{alg:ln:lemma_linvcd_end}). In \figref{fig:warmuth_lemma_run}, we provide an illustrative example  for creating such partitions for the Warmuth hypothesis class from \tableref{tab:warmth_example_sequences}. We sketch the proof of \lemref{lem:main:div_concepts} below, and defer the detailed proof 
\iftoggle{longversion}{to Appendix~\ref{sec.appendix.seq-models_linearVCD.lemmaproof}.
}
{to Appendix~D.1.
}
\begin{proof}{[Proof Sketch of \lemref{lem:main:div_concepts}]}
\looseness -1 Let us define 
$\hypotheses_{\instance} = \{\hypothesis \in \hypotheses: {\hypothesis \triangle \instance}_{|\compactDSet{{\hypotheses}}} \in \hypotheses_{|\compactDSet{{\hypotheses}}}\}$.
Here, $\hypothesis \triangle \instance$ denotes the hypothesis that only differs with $\hypothesis$ on the label of $\instance$, and $\hypothesis_{|\compactDSet{{\hypotheses}}}$ denotes the patterns of $\hypothesis$ on $\compactDSet{{\hypotheses}}$. Fix a reference hypothesis $\hypothesis_\hypotheses$. For all $j \in [m-1]$, let $\clabel_j = 1 -  \hypothesis_\hypotheses(\instance_j)$ be the opposite label of $\instance_j \in \compactDSet{\hypotheses}$ as provided by $\hypothesis_\hypotheses$.  As shown in \lnref{alg:linvcd:partition} of \algref{alg:lemma_linVCD}, we consider the set $\hclass{1} := \hypotheses^{\clabel_1}_{\instance_1} = \{\hypothesis \in \hypotheses_{\instance_1}: \hypothesis(\instance_1) = \clabel_1\}$ as the first partition. In the appendix, we show that $|\hclass{1}| > 0$.


Next, we show that $\VCD(\hclass{1}, \compactDSet{\hclass{}}\setminus\{\instance_1\}) \leq d-1$.
When $d>1$, we prove the statement as follows:
\begin{equation*}
     \VCD(\hclass{1}, \compactDSet{\hypotheses} \setminus\{\instance_1\}) \leq \VCD(\hypotheses^{\clabel_1}_{\instance_1}, \compactDSet{\hypotheses}) = \VCD(\hypotheses_{\instance_1}, \compactDSet{\hypotheses}) - 1 \leq \VCD(\hypotheses, \compactDSet{\hypotheses}) - 1 \leq d - 1
\end{equation*}
%
In the appendix, we prove the statement for $d=1$, and further show that there exists a compact-distinguishable set $\compactDSet{\hclass{1}} \subseteq \compactDSet{\hypotheses} \setminus\{\instance_1\}$ for the first partition $\hclass{1}$.
%
Then, we conclude that the first partition $\hclass{1}$ has $\VCD(\hclass{1}, \compactDSet{\hclass{1}}) \leq d - 1$.

Next, we remove the first partition $\hclass{1}$ from $\hypotheses$, and continue to create the above mentioned partitions on $\hypotheses_{\text{rest}} = \hypotheses \setminus \hclass{1}$ and $\instances_{\text{rest}} = \compactDSet{\hypotheses} \setminus \{\instance_1\}$. 
As discussed in the appendix, 
we show that $\instances_{\text{rest}}$ is a compact-distinguishable set on $\hypotheses_{\text{rest}}$. 
Therefore, we can repeat the above procedure (\lnref{alg:ln:lemma_linvcd_start}-- \lnref{alg:ln:lemma_linvcd_end}, \algref{alg:linvcd}) to create the subsequent partitions. This process continues until the size of $\instances_{\text{rest}}$ reduces to $1$, i.e. $\instances_{\text{rest}} = \{\instance_{m-1}\}$. Until then, we obtain partitions $\{\hclass{1}, ..., \hclass{m-2}\}$. By construction, $\hclass{j}$ satisfy properties (i) and (ii) for all $j \in [m-2]$.

\looseness -1 It remains to show that $\hclass{m-1}$ and $\hclass{m}$ also satisfy the properties in \lemref{lem:main:div_concepts}. Since $\instances_{\text{rest}} = \{\instance_{m-1}\}$ before we start iteration $m-1$, and $\instances_{\text{rest}}$ is a compact-distinguishable set for 
$\hypotheses_{\text{rest}}$, there must exist exactly two hypotheses in $\hypotheses_{\text{rest}}$, and therefore $|\hclass{m-1}|, |\hclass{m}|=1$. This implies that $\VCD(\hclass{m-1}, \compactDSet{\hclass{m-1}}) = \VCD(\hclass{m}, \compactDSet{\hclass{m}}) = 0$. Furthermore, $\forall j \in [m-1]$ and $\hypothesis \in \hclass{j}$, we have  $\hypothesis_\hypotheses(\instance_j) \neq \hypothesis(\instance_j)$. This indicates $\hypothesis_\hypotheses \in \hypotheses_m$, and hence $\hypotheses_m = \{\hypothesis_\hypotheses\}$ which completes the proof.
\end{proof}
\newcommand*{\tikzmk}[1]{\tikz[remember picture,overlay,] \node (#1) {};\ignorespaces}
\newcommand{\boxit}[1]{\tikz[remember picture,overlay]{\node[yshift=3pt,fill=#1,opacity=.25,fit={($(A)-(.08\linewidth,-.2\baselineskip)$)($(B)+(\linewidth,.8\baselineskip)$)}] {};}\ignorespaces}
\begin{figure}[t!]
\centering
\makeatletter
    \renewcommand{\ALG@name}{Algorithm}
    \makeatother
\scalebox{0.9}{
  \begin{minipage}[b]{0.59\textwidth}
    \begin{algorithm}[H]
      \caption{Recursive procedure for constructing $\pref_\lvs$ achieving $\TD_{\Instances,\Hypotheses,\hinit}(\pref_\lvs) \leq \VCD(\Hypotheses, \Instances)$}\label{alg:linvcd}
          \hspace*{\algorithmicindent} \textbf{Input:} $\Instances$, $\Hypotheses$, $\hinit$
        \begin{algorithmic}[1]
            \State Let $I : \Hypotheses \rightarrow \{1, \dots, |\Hypotheses|\}$ be any bijective mapping 
            \State For all $\hypothesis' \in \Hypotheses$, $\hypotheses \subseteq \Hypotheses$, $\hypothesis \in \Hypotheses$, initialize 
            \[\pref_\lvs(\hypothesis'; \hypotheses, \hypothesis) \leftarrow
            \begin{cases}
            0 & \text{if}~\hypothesis' = \hypothesis\\
            |\Hypotheses| + 1 & \text{o.w.}
            \end{cases}
            \]\label{alg:ln:init_pref}
            \State $\textsc{SetPreference}(\Hypotheses, \Hypotheses, \Instances, \hinit)$
            \Function{SetPreference}{$V, {\hypotheses}, {\instances}, \hypothesis$}
            \State Create compact-distinguishable set $\compactDSet{{\hypotheses}} \subseteq {\instances}$
            \State $\hypotheses_{\text{rest}} := {\hypotheses},  \instances_{\text{rest}} := \compactDSet{{\hypotheses}}$
            \For{\tikzmk{A} $\instance \in \compactDSet{{\hypotheses}}$
            }
            \label{alg:ln:lemma_linvcd_start}
                \State $\clabel = 1 - \hypothesis(\instance)$
                %
                \State ${\hypotheses^{\clabel}_{\instance} \leftarrow \{\hypothesis' \in {\hypotheses_{\text{rest}}} : {\hypothesis' \triangle \instance}_{|{\instances_{\text{rest}}}}  \in {\hypotheses_{\text{rest}}}_{|{\instances_{\text{rest}}}}, \hypothesis'(\instance) = \clabel\}}$\label{alg:linvcd:partition}
                \State ${{\hypotheses_{\text{rest}}} \leftarrow {\hypotheses_{\text{rest}}} \setminus \hypotheses^{\clabel}_{\instance} }$,
                \label{alg:ln:lemma_linvcd_end}
                ${\instances_{\text{rest}}} \leftarrow {\instances_{\text{rest}}} \setminus\{x\}$
                \State 
                \tikzmk{B}\boxit{black!30} 
                $V_{\text{next}} \leftarrow V \cap \Hypotheses(\{(\instance,\clabel)\})$ 
                \State \textbf{for} {$\hypothesis' \in \hypotheses^{\clabel}_{\instance}$}
                    \textbf{ do } $\pref_\lvs(\hypothesis'; V_{\text{next}}, \hypothesis) \leftarrow \indexof{\hypothesis'}$\label{alg:ln:set_pref_value}
                \State $\hnext \leftarrow \argmin_{{\hypothesis'} \in \hypotheses^{\clabel}_{\instance}} \indexof{\hypothesis'}$ \label{alg:ln:assign_h_next}
                \State
                $\textsc{SetPreference}(V_{\text{next}}, \hypotheses^{\clabel}_{\instance}, \compactDSet{{\hypotheses}} \setminus\{x\}, \hnext)$\label{alg:ln:recursion}
            \EndFor
            \EndFunction
    	\end{algorithmic}\label{alg:lemma_linVCD}
    \end{algorithm}
    \end{minipage}}
    \hfill
    \scalebox{0.9}{
    \begin{minipage}[b]{0.4\textwidth}
    \begin{figure}[H]
    \scalebox{0.70}{
    \hbox{\hspace{-1cm}
    \begin{tikzpicture}
  [every matrix/.append style={ampersand replacement=\&,matrix of nodes},
  level 1/.style={sibling distance=3.5cm},
  nodes={anchor=west}]
  \node [matrix,draw] at (-1,0) (m0) {1 \& 1 \& 0 \& 0 \& 0 \\}
  child {node[matrix,draw] (m1) at (3,0) {0 \& 0 \& 1 \& 1 \& 0 \\} edge from parent node[left,yshift=1mm] {$(x_1, 0)$}}
  child {node[matrix,draw] (m2) at (1.1,-2)
    {0 \& 0 \& 0 \& 1 \& 1 \\
      1 \& 0 \& 1 \& 0 \& 1 \\
    }
    edge from parent node[left,yshift=-2mm] {$(x_2, 0)$}
  }
  child {node[matrix,draw] (m3) at (-1,-4.5)
    {0 \& 1 \& 1 \& 0 \& 0 \\
      1 \& 0 \& 1 \& 1 \& 0 \\
      0 \& 1 \& 1 \& 0 \& 1 \\
    }
    edge from parent node[right,yshift=3mm] {$(x_3, 1)$};
  }
  child {node[matrix,draw] (m4) at (-2.5,-2)
    {1 \& 1 \& 0 \& 1 \& 0 \\
      0 \& 1 \& 0 \& 1 \& 1 \\}
    edge from parent node[above,yshift=2mm,xshift=3mm] {$(x_4, 1)$}
  }
  child {node[matrix,draw] (m5) at (-4.6, 0) {1 \& 0 \& 0 \& 0 \& 1 \\} edge from parent node[above] {$(x_5, 1)$}};

  \node[left,xshift=-3mm] at (m0-1-1) {$h_1$};
  \node[left,xshift=-3mm] at (m1-1-1) {$h_3$};
  \node[left,xshift=-3mm] at (m2-1-1) {$h_4$};
  \node[left,xshift=-3mm] at (m2-2-1) {$h_{10}$};
  \node[left,xshift=-3mm] at (m3-1-1) {$h_2$};
  \node[left,xshift=-3mm] at (m3-2-1) {$h_8$};
  \node[left,xshift=-3mm] at (m3-3-1) {$h_7$};
  \node[left,xshift=-3mm] at (m4-1-1) {$h_6$};
  \node[left,xshift=-3mm] at (m4-2-1) {$h_9$};
  \node[left,xshift=-3mm] at (m5-1-1) {$h_5$};

  \node[above,yshift=3mm] at (m0) {$H^{6}$};
  \node[above left,yshift=3mm] at (m1) {$H_{x_1}^0$};
  \node[above left,yshift=5mm] at (m2) {$H_{x_2}^0$};
  \node[above left,yshift=8mm] at (m3) {$H_{x_3}^1$};
  \node[above,yshift=5mm] at (m4) {$H_{x_4}^1$};
  \node[above,yshift=3mm] at (m5) {$H_{x_5}^1$};
  
  \scoped[on background layer]
  {
    \node[fill=black!10, fit=(m1-1-1)(m1-1-1) ]   {};
    \node[fill=black!10, fit=(m2-1-2)(m2-2-2) ]   {};
    \node[fill=black!10, fit=(m3-1-3)(m3-3-3) ]   {};
    \node[fill=black!10, fit=(m4-1-4)(m4-2-4) ]   {};
    \node[fill=black!10, fit=(m5-1-5)(m5-1-5) ]   {};
  }
\end{tikzpicture}}}
\caption{Illustration of \lemref{lem:main:div_concepts} on the Warmuth class. The grouped hypotheses in the leaf clusters correspond to the sets $\hypotheses^{\clabel}_{\instance}$ created in \lnref{alg:linvcd:partition} of \algref{alg:linvcd}.}\label{fig:warmuth_lemma_run}
\end{figure}
\end{minipage}
}
\end{figure}

\begin{figure}[t!]
    \centering
    \scalebox{0.66}{
    \begin{tikzpicture}[
  every matrix/.append style={ampersand replacement=\&,matrix of nodes},
  subordinate/.style={%
    grow=down,
    xshift=-3.2em, 
    text centered, text width=12em,
    edge from parent path={(\tikzparentnode.205) |- (\tikzchildnode.west)}
  },
  level1/.style ={level distance=4em,anchor=north},
  level2/.style ={level distance=8em,anchor=north},
  level 1/.style={edge from parent fork down,sibling distance=10em,level distance=5em},
  level 2/.style={edge from parent fork down,sibling distance=10em},  
  nodes={anchor=west}]
  \node [matrix,draw] (m0) {1 \& 1 \& 0 \& 0 \& 0 \\}
  child[level 1] {node[matrix,draw] (m1) 
    {0 \& 0 \& 1 \& 1 \& 0 \\} edge from parent node[above] {$(x_1, 0)$}}
  child[level 1] {node[matrix,draw] (m2) 
    {0 \& 0 \& 0 \& 1 \& 1 \\
    }
    child[level 2] {node[matrix,draw] (m6) at (-3,0)
      {
        1 \& 0 \& 1 \& 0 \& 1 \\
      }
      edge from parent node[above] {$(x_3, 1)$}
    }
    edge from parent node[above] {$(x_2, 0)$}
  }
  child[level 1] {node[matrix,draw] (m3) 
    {
      0 \& 1 \& 1 \& 0 \& 0 \\
    }
    child[level 2]{node[matrix,draw] (m7) at (-1,0)
      {1 \& 0 \& 1 \& 1 \& 0\\}
      edge from parent node[above] {$(x_4, 1)$}
    }
    child[level 2]{node[matrix,draw] (m8) at (-1.5,0)
      {
        0 \& 1 \& 1 \& 0 \& 1 \\
      }
      edge from parent node[above] {$(x_5, 1)$}
    }
    edge from parent node[above] {$(x_3, 1)$}
  }
  child[level 1] {node[matrix,draw] (m4) 
    {1 \& 1 \& 0 \& 1 \& 0 \\}
    child[level 2] {node[matrix,draw] (m9) 
      { 0 \& 1 \& 0 \& 1 \& 1 \\}
      edge from parent node[above] {$(x_5, 1)$}
    }
    edge from parent node[above] {$(x_4, 1)$}
  }
  child[level 1] {node[matrix,draw] (m5) 
    {1 \& 0 \& 0 \& 0 \& 1 \\}
    edge from parent node[above] {$(x_5, 1)$}
  };
  \node[left,xshift=-3mm] at (m0-1-1) {$h_1$};
  \node[left,xshift=-3mm] at (m1-1-1) {$h_3$};
  \node[left,xshift=-3mm] at (m2-1-1) {$h_4$};
  \node[left,xshift=-3mm] at (m6-1-1) {$h_{10}$};
  \node[left,xshift=-3mm] at (m3-1-1) {$h_2$};
  \node[left,xshift=-3mm] at (m7-1-1) {$h_8$};
  \node[left,xshift=-3mm] at (m8-1-1) {$h_7$};
  \node[left,xshift=-3mm] at (m4-1-1) {$h_6$};
  \node[left,xshift=-3mm] at (m9-1-1) {$h_9$};
  \node[left,xshift=-3mm] at (m5-1-1) {$h_5$};
  %
  %
\end{tikzpicture}}
    \caption{Illustration of \thmref{thm:main:seq-models_vs_VCD} proof -- constructing a  $\sigma_\lvs \in \SigmaLvs$ for the Warmuth class.}
    \vspace{-8mm}
    \label{fig:warmuth-lvs-theorem-construction}
\end{figure}

\paragraph{Recursive construction of $\pref_\lvs$.} As a part of the \thmref{thm:main:seq-models_vs_VCD} proof, we provide a recursive procedure for constructing a $\sigma_\lvs \in \SigmaLvs$ achieving $\TD_{\Instances,\Hypotheses,\hinit}(\pref_\lvs) = \bigO{\VCD(\Hypotheses,\Instances)}$.

\begin{proof}[Proof of \thmref{thm:main:seq-models_vs_VCD}]
In a nutshell, the proof consists of three steps: (i) initialization of $\pref_\lvs$, (ii) setting the preferences by recursively invoking the constructive procedure for \lemref{lem:main:div_concepts}, and (iii) showing that there exists a teaching sequence of length up to $d$ for any target hypothesis $\hstar$. 
We summarize the recursive procedure in \algref{alg:lemma_linVCD}. 

\textit{\underline{Step (i).}}
To begin with, we initialize $\pref_\lvs$ with default values which induce high $\sigma$ values (i.e., low preference), except for $\sigma(\hypothesis';\hypotheses,\hypothesis)=0$ where $\hypothesis'=\hypothesis$ (c.f. \lnref{alg:ln:init_pref} of \algref{alg:lemma_linVCD}). The self-preference guarantees that $\sigma_\lvs$ is collusion-free as per Definition~\ref{def:seq-col-free}. 

\textit{\underline{Step (ii).}}
The recursion begins at the top level with $\hypotheses = \Hypotheses$, current version space $V = \Hypotheses$, and initial hypothesis $\hypothesis=\hinit$. \lemref{lem:main:div_concepts} suggests that we can partition $\hypotheses$ into $m=|\compactDSet{\hypotheses}|+1$ groups $\{\hclass{1}, ..., \hclass{m}\}$, where for all $j \in [m]$, there exists a compact-distinguishable set $\compactDSet{\hclass{j}}$ that satisfies the properties in \lemref{lem:main:div_concepts}.

Now consider the hypothesis $\hypothesis := \hinit$. 
We show that for $j \in [m - 1]$, every $(\instance_j, \clabel_j)$, where $\instance_j\in \compactDSet{\hypotheses}$ and $\clabel_j=1-\hypothesis(x_j)$, corresponds to a unique version space $V^j:= \{h \in V : h(\instance_j) =  \clabel_j\}$. To prove this statement, we consider $R^j := {V}^j \cap \hypotheses = \{\hypothesis \in \hypotheses : \hypothesis(\instance_j) = \clabel_j\} $.
\iftoggle{longversion}{According to Lemma~\ref{lem:app:vs_dependent_TD} of Appendix~\ref{sec.appendix.seq-models_linearVCD.theoremproof},}
{As is discussed in Appendix~D.2 of the supplementary,} we know that none of $R^j$ for $j \in [m-1]$ are equal. This indicates that none of $V^j$ for $j \in [m-1]$ are equal.  

We then set the values of the preference function $\pref_\lvs(\cdot; V^j, h)$ for all $j \in [m - 1]$ and $\clabel_j = 1 - \hypothesis(\instance_j)$ (\lnref{alg:ln:set_pref_value}). Upon receiving $(\instance_j, \clabel_j)$, the learner will be steered to the next ``search space'' $\hclass{j}$, with version space $V^j$. By \lemref{lem:main:div_concepts} we have $\VCD(\hclass{j}, \compactDSet{\hclass{j}})\leq \VCD(\hypotheses, \compactDSet{\hypotheses}) - 1$.

We will build the preference function $\pref_\lvs$ recursively $m - 1$ times for each 
$( V^j, \hclass{j}, \compactDSet{\hclass{j}}, \hypothesis_{\text{next}})$, where $\hypothesis_{\text{next}}$ corresponds to the unique hypothesis identified by function $I$ (\lnref{alg:ln:assign_h_next}--\lnref{alg:ln:recursion}).
At each level of recursion, \VCD reduces by 1. We stop the recursion when $\VCD(\hclass{j}; \compactDSet{\hclass{j}})=0$, 
which corresponds to the scenario $|\hclass{j}|=1$. 

\textit{\underline{Step (iii).}}
Given the preference function constructed in \algref{alg:linvcd}, we can build up the set of (labeled) teaching examples recursively. 
Consider the beginning of the teaching process, where the learner's current hypothesis is $\hinit$ and version space is $\Hypotheses$, and the goal of the teacher is to teach $\hstar$. Consider the first level of the recursion in \algref{alg:lemma_linVCD}, where we divide $\Hypotheses$ into $m=|\compactDSet{\Hypotheses}|+1$ groups $\{\hclass{1}, ..., \hclass{m}\}$. Let us consider the case where $\hstar \in \hclass{j^\star}$ with $j^\star \in [m-1]$. 
The teacher provides an example given by $(\instance=\instance_{j^\star}, \clabel=\hstar(\instance_{j^\star}))$. After receiving the teaching example, the resulting partition $\hclass{j^\star}$ will stay in the version space; meanwhile, $\hinit$ will be removed from the version space. The new version space will be $V^{j^\star}$. The learner's new hypothesis induced by the preference function is given by $\hypothesis_{\text{next}} \in \hclass{j^\star}$. By repeating this teaching process for a maximum of $d$ steps, the learner reaches a partition of size 1 (see \emph{{Step (ii)}} for details). At this step $\hstar$ must be the only hypothesis left in the search space. Therefore, $\hypothesis_{\text{next}} = \hstar$, and the learner has reached $\hstar$. 
%
%
\end{proof}
\figref{fig:warmuth-lvs-theorem-construction} illustrates the recursive construction of a $\pref_\lvs \in \SigmaLvs$ for the Warmuth class, with $\TD_{\Instances,\Hypotheses,\hinit}(\pref_\lvs) = 2$.



%% file: 6_conclusion.tex
\vspace{-1mm}
\section{Discussion and Conclusion}\label{sec:discussion}
\vspace{-1mm}
We now discuss a few thoughts related to different families of preference functions. First of all, the size of the families grows exponentially as we change our model from $\SigmaConst$, $\SigmaGlobal$ to $\SigmaGvs$/$\SigmaLocal$ and finally to $\SigmaLvs$, thus resulting in more powerful models with lower teaching complexity. While run time has not been the focus of this paper, it would be interesting to characterize the presumably increased run time complexity of sequential learners and teachers  with complex preference functions. Furthermore, as the size of the families grow, the problem of finding the best preference function $\sigma$ in a given family $\Sigma$ that achieve the minima in Eq.~\eqref{eq:sigmatd} becomes more computationally challenging.

The recursive procedure in \algref{alg:lemma_linVCD} creates a preference function $\pref_\lvs\in\SigmaLvs$ that has teaching complexity at most $\VCD$. It is interesting to note that the resulting preference function $\pref_\lvs$ has the characteristic of ``win-stay, loose shift" \cite{bonawitz2014win,chen2018understanding}: Given that for any hypothesis we have $\sigma(h;\cdot,h) = 0$, the learner prefers her current hypothesis as long as it remains consistent. Preference functions with this characteristic naturally exhibit the collusion-free property in \defref{def:seq-col-free}. For some problems, one can achieve lower teaching complexity for a
$\sigma \in \SigmaLvs$. 
In fact, the preference function $\sigma_{\lvs}$ we provided for the Warmuth class in \tableref{tab:pref_lvs} has teaching complexity $1$, while the preference function constructed in \figref{fig:warmuth-lvs-theorem-construction} has teaching complexity $2$.

One fundamental aspect of modeling teacher-learner interactions is the notion of collusion-free teaching. Collusion-freeness for the batched setting is well established in the research community and \NCTD characterizes the complexity of the strongest collusion-free batch model. In this paper, we are introducing a new notion of collusion-freeness for the sequential setting (\defref{def:seq-col-free}). As discussed above, a stricter condition is the ``win-stay lose-shift'' model, which is easier to validate without running the teaching algorithm. In contrast, the condition of \defref{def:seq-col-free} is more involved in terms of validation and is a joint property of the teacher-learner pair. One intriguing question for future work is defining notions of collusion-free teaching in sequential models and understanding their implications on teaching complexity.


Another interesting direction of future work is to better understand the properties of the teaching parameter $\Sigmatd{}$. One question of particular interest is showing that the teaching parameter is not upper bounded by any constant independent of the hypothesis class, which would suggest a strong collusion in our model. We can show that for certain hypothesis classes, $\Sigmatd{}$ is lower bounded by a function of $\VCD$. In particular, for the power set class of size $d$ (which has $\VCD=d$),  $\Sigmatd{}$ is lower bounded by $\bigOmega{\frac{d}{\log d}}$. Another direction of future work is to understand whether this parameter is additive or subadditive over disjoint domains. Also, we consider a generalization of our results to the infinite VC classes as a very interesting direction for future work.

Our framework provides novel tools for reasoning about teaching complexity by constructing preference functions. This opens up an interesting direction of research to tackle important open problems, such as proving whether \NCTD or \RTD is linear in \VCD~\cite{simon2015open,chen2016recursive,DBLP:journals/corr/HuWLW17,pmlr-v98-kirkpatrick19a}. In this paper, we showed that neither of the families $\SigmaGvs$ and $\SigmaLocal$ dominates the other (Theorem~\ref{thm:local-eq-GVS}). As a direction for future work, it would be important to further quantify the complexity of $\SigmaLocal$ family.

%% file: 6.1_acknowledgments.tex
\subsubsection*{Acknowledgements}
This work was done in part when Yuxin Chen was at Caltech. Xiaojin Zhu is supported by NSF 1545481, 1561512, 1623605, 1704117, 1836978 and the MADLab AF CoE FA9550-18-1-0166.

%% file: 7.1_appendix_batch-models.tex
\section{Supplementary Materials for \secref{sec:non-seq-family-pref}} \label{sec.appendix.batch-models}

\subsection{An Example Hypothesis Class and the Teaching Sequences for the Batch Models}\label{sec.appendix.batch_models.example}
In this section, we provide an example hypothesis class where $\SigmaConstTD=\TD=3$, $\SigmaGlobalTD=\RTD=2$, and $\SigmaGvsTD=\NCTD=1$. The hypothesis class is specified in \tableref{tab:h2}. The preference functions inducing the optimal teaching sets for the examples are specified in Table~\ref{tab:h2_pref_const}, \ref{tab:h2_pref_glbl}, and \ref{tab:h2_pref_gvs}.

\begin{table}[h!]
    \centering
    \begin{subtable}[t]{\textwidth}
        \centering
        \begin{tabular}{c|cccccc||c|c|c}
        \backslashbox{$\Hypotheses$}{$\Instances$}
        & $\instance_1$ & $\instance_2$ & $\instance_3$ & $\instance_4$ & $\instance_5$ & $\instance_6$ & $\TeachingSeq_{\textsf{const}}$ & $\TeachingSeq_{\textsf{\glbl}}$ & $\TeachingSeq_{\gvs}$\\ 
        \hline
        $\hypothesis_1$ & 1 & 0 & 0 & 0 & 0 & 1 & $\paren{\instance_1, \instance_6}$ & $\paren{\instance_1, \instance_6}$  & $\paren{\instance_1}$\\
        $\hypothesis_2$ & 0 & 1 & 0 & 0 & 0 & 1 & $\paren{\instance_2, \instance_6}$ & $\paren{\instance_2, \instance_6}$ & $\paren{\instance_2}$\\
        $\hypothesis_3$ & 1 & 1 & 1 & 0 & 0 & 0 & $\paren{\instance_3, \instance_4, \instance_5}$ & $\paren{\instance_1}$ & $\paren{\instance_3}$\\
        $\hypothesis_4$ & 1 & 1 & 1 & 1 & 0 & 0 & $\paren{\instance_4, \instance_5}$ & $\paren{\instance_4, \instance_5}$ & $\paren{\instance_4}$\\
        $\hypothesis_5$ & 1 & 1 & 1 & 0 & 1 & 0 & $\paren{\instance_4, \instance_5}$ & $\paren{\instance_4, \instance_5}$ & $\paren{\instance_5}$\\
        $\hypothesis_6$ & 0 & 0 & 0 & 1 & 1 & 1 & $\paren{\instance_4, \instance_5}$ & $\paren{\instance_4, \instance_5}$ & $\paren{\instance_6}$
        \end{tabular}
        \caption{An example hypothesis class with the optimal teaching sets under different families of preference functions.}\label{tab:h2}
    \end{subtable}
    \begin{subtable}[t]{.45\textwidth}
        \begin{tabular}{c|cccccc}
    $\hypothesis'$ & $\hypothesis_1$ & $\hypothesis_2$ & $\hypothesis_3$ & $\hypothesis_4$ & $\hypothesis_5$ & $\hypothesis_6$ \\\hline
        $\sigma_{\textsf{const}}(\hypothesis'; \cdot, \cdot)$ & 0 & 0 & 0 & 0 & 0 & 0\\
    \end{tabular}
    \caption{Preference function $\sigma_{\textsf{const}}$} \label{tab:h2_pref_const}
    \end{subtable}
    \quad \qquad
    \begin{subtable}[t]{.45\textwidth}
        \begin{tabular}{c|cccccc}
    $\hypothesis'$ & $\hypothesis_1$ & $\hypothesis_2$ & $\hypothesis_3$ & $\hypothesis_4$ & $\hypothesis_5$ & $\hypothesis_6$ \\\hline
         $\sigma_{\textsf{global}}(\hypothesis'; \cdot, \cdot)$ & 1 & 1 & 0 & 1 & 1 & 1\\
    \end{tabular}
    \caption{Preference function $\sigma_{\glbl}$}\label{tab:h2_pref_glbl}
    \end{subtable}
    \begin{subtable}[t]{\textwidth}
    \begin{tabular}{c|cccccc}
    $\hypothesis'$ & $\hypothesis_1$ & $\hypothesis_2$ & $\hypothesis_3$ & $\hypothesis_4$ & $\hypothesis_5$ & $\hypothesis_6$ \\\hline
         \multirow{5}{*}{$\hypotheses$} & 
         $\{\hypothesis_1, \hypothesis_3, \hypothesis_4, \hypothesis_5\}$ & 
         $\{\hypothesis_2, \hypothesis_3, \hypothesis_4, \hypothesis_5\}$ &
         $\{\hypothesis_3, \hypothesis_4, \hypothesis_5\}$ & 
         $\{\hypothesis_4, \hypothesis_6\}$ & 
         $\{\hypothesis_5, \hypothesis_6\}$ & 
         $\{\hypothesis_1, \hypothesis_2, \hypothesis_6\}$ 
         \\ 
        & 
        $\{\hypothesis_1, \hypothesis_3, \hypothesis_4\}$ & 
        $\{\hypothesis_2, \hypothesis_3, \hypothesis_4\}$ &
         $\{\hypothesis_3, \hypothesis_4\}$ & 
         $\{\hypothesis_4\}$ & 
         $\{\hypothesis_5\}$ & 
         $\{\hypothesis_1, \hypothesis_6\}$ 
         \\
         & 
        $\{\hypothesis_1, \hypothesis_3, \hypothesis_5\}$ & 
        $\{\hypothesis_2, \hypothesis_3, \hypothesis_5\}$ &
         $\{\hypothesis_3, \hypothesis_5\}$ & 
          & 
          & 
         $\{\hypothesis_2, \hypothesis_6\}$ 
         \\
          & 
        $\{\hypothesis_1\}$ & 
        $\{\hypothesis_2\}$ &
         $\{\hypothesis_3\}$ & 
          & 
          & 
         $\{\hypothesis_6\}$ 
         \\\hline
         $\sigma_\gvs(\hypothesis'; \hypotheses, \cdot)$ & 0 & 0 & 0 & 0 & 0 & 0\\
    \end{tabular}
    \caption{Preference function $\sigma_\gvs$. For all other $\hypothesis', \hypotheses$ pairs not specified in the table, $\sigma(\hypothesis', \hypotheses, \cdot) = 1$.}\label{tab:h2_pref_gvs}
    \end{subtable}
    \caption{An example hypothesis class where $\SigmaConstTD=3$, $\SigmaGlobalTD=2$, and $\SigmaGvsTD=1$. 
    }
    \label{tab:batch_model_example_h2}
\end{table}

\subsection{Proof of \thmref{theorem:equivelence-results}}\label{sec.appendix.batch-models.theorem-proof}

Before we prove our main results for the batch models, we first establish the following results on the non-clashing teaching. The notion of a non-clashing teacher was first introduced by \cite{KKW07-nonclashing}. Our proof is inspired by \cite{pmlr-v98-kirkpatrick19a} which shows the non-clashing property for collusion-free teacher-learner pair, under the batch setting. 
\begin{lemma}\label{lm:seq-vs-bat_coll-free} 
 Assume $\sigma \in \SigmaGvs$ is collusion-free. Then teacher $\Teacher$ must be non-clashing on $\Hypotheses$. i.e., for any two distinct $\hypothesis, \hypothesis' \in \Hypotheses$ such that $\Teacher(\hypothesis)$ is consistent with $\hypothesis'$, $\Teacher(\hypothesis')$ cannot be consistent with $\hypothesis$.
\end{lemma}
\begin{proof}[Proof of \lemref{lm:seq-vs-bat_coll-free}] 
By definition of the preference function, we have  $\forall \sigma\in \SigmaGvs, \hypothesis'\in \Hypotheses$, $\sigma(\hypothesis'; \Hypotheses(\examples'), \cdot) = g_\sigma(\hypothesis', \Hypotheses(\examples'))$ for some function $g_\sigma$.

We then prove the lemma by contradiction. Assume that the teacher mapping $\Teacher$ isn't non-clashing. There exists $\hypothesis \neq \hypothesis' \in \Hypotheses$, where $\examples = \Teacher(\hypothesis)$, and $\examples' = \Teacher(\hypothesis')$ are consistent with both, $\hypothesis$ and $\hypothesis'$.

Assume that the last current hypothesis before the teacher provides the last example of $\examples$ is $\hypothesis_1$. Then,
\begin{equation*}
    \hypothesis = \argmin_{\hypothesis'' \in \Hypotheses(\examples)} \sigma(\hypothesis'';\Hypotheses(\examples),\hypothesis_1) = \argmin_{\hypothesis'' \in \Hypotheses(\examples\cup\examples')} \sigma(\hypothesis''; \Hypotheses(\examples \cup \examples'), \hypothesis_1) = \argmin_{\hypothesis'' \in \Hypotheses(\examples \cup \examples')} g_\sigma(\hypothesis'', \Hypotheses(\examples \cup \examples')).
\end{equation*}
Where first equality is the definition of a teaching sequence. The second equality is by the definition of collusion-free preference (\defref{def:seq-col-free}). 
Similarly we have
\begin{equation*}
    \hypothesis' =  \argmin_{\hypothesis'' \in \Hypotheses(\examples' \cup \examples)} g_\sigma(\hypothesis'', \Hypotheses(\examples' \cup \examples)).
\end{equation*}
Consequently, $\hypothesis = \hypothesis'$, which is a contradiction. This indicates that $\Teacher$ is non-clashing.
\end{proof}

Now we are ready to provide the proof for \thmref{theorem:equivelence-results}. We divide the proof of the \thmref{theorem:equivelence-results} into three parts, each corresponding to the equivalence results for a different preference function family.
\begin{proof}[Proof of \thmref{theorem:equivelence-results}] 
Part 1 (reduction to \TD) and Part 2 (reduction to \RTD) of the proof are included in the main paper.


To establish the equivalence between $\SigmaGvsTD$ and $\NCTD$, we aim to show that for any hypotheses space $\Hypotheses$, it holds (i) $\SigmaGvsTD_{\Instances,\Hypotheses,\hinit} \geq \NCTD(\Hypotheses)$, and (ii) $\SigmaGvsTD_{\Instances,\Hypotheses,\hinit} \leq \NCTD(\Hypotheses)$.

We first prove (i). 
According to \lemref{lm:seq-vs-bat_coll-free}, for any $\pref \in \Sigma_{\gvs}$, a successful teacher $\Teacher$ with $\pref$  is non-clashing on $\Hypotheses$. Therefore, $\SigmaGvsTD_{\Instances,\Hypotheses,\hinit} =   \min_{\text{Successful Teacher $\Teacher$}}\max_{\hypothesis \in \Hypotheses}  |\Teacher(\hypothesis)| \geq \min_{\text{Non-clashing teacher~} \Teacher} \max_{h\in \Hypotheses}|\Teacher(\hypothesis)| = \NCTD(\Hypotheses)$.

 
We now proceed to prove (ii). Consider any non-clashing teacher mapping $\Teacher$. 
First we will prove that there exists $\pref \in \Sigma_\gvs$ such that $(\Teacher, L_\sigma)$ is successful on $\Hypotheses$. Here $L_\sigma$ is a learner corresponded to $\sigma$ as described in \secref{sec:model}, and by ``successful'' we mean that the learner successfully outputs the target hypothesis when teaching terminates. 
In the following, we construct a preference function $\sigma$. First initialize $\sigma(\cdot;\cdot,\cdot) = 1$. Then, for every $\hypothesis \in \Hypotheses$, and every $S'$ which $\Teacher(h) \subseteq S'$ and $S'$ is consistent with $\hypothesis$ assign $\sigma(\hypothesis; \Hypotheses(S'),\cdot) = 0$.

We then prove (ii) by contradiction. Consider any set of examples $S$, and assume there exists two $\hypothesis' \neq \hypothesis' \in \Hypotheses$ where $\sigma(\hypothesis; \Hypotheses(S),\cdot) = \sigma(\hypothesis'; \Hypotheses(S),\cdot) = 0$. Then $\Teacher(\hypothesis) \subseteq \Hypotheses(S)$ and $\Teacher(\hypothesis') \subseteq \Hypotheses(S)$, also $S$ is consistent with both $\hypothesis$ and $\hypothesis'$. This indicates that, both $\Teacher(\hypothesis)$ and $\Teacher(\hypothesis')$ must be consistent with both $\hypothesis$, and $\hypothesis'$. This contradicts with $T$ being non-clashing. Therefore, for every $\hypothesis$, and $S'$ where $S'$ is consistent with $\hypothesis$ and $\Teacher(\hypothesis) \subseteq S'$, and $\hypothesis' \neq \hypothesis$, we have $\sigma(\hypothesis; \Hypotheses(S'),\cdot) < \sigma(\hypothesis'; \Hypotheses(S'),\cdot)$. Consequently, after providing the examples $\Teacher(\hypothesis)$ to the learner $L_\sigma$, the learner will stay on $\hypothesis$ even if she receives more consistent labeled examples. Therefore, $(\Teacher,L_\sigma)$ is both collusion-free and successful on $\Hypotheses$. 

Therefore, we conclude that for any teacher mapping $T$ induced by $\sigma\in \SigmaGvs$, $\max_{\hypothesis\in\Hypotheses}|\Teacher(h)| \geq \TD_{\Instances,\Hypotheses,\hinit}(\pref)$. Consequently, $\SigmaGvsTD_{\Instances,\Hypotheses,\hinit} \leq \NCTD(\Hypotheses)$. Combining this results with (i) hence completes the proof. \end{proof}

%% file: 7.2_appendix_seq-models_warmuth.tex
\section{Supplementary Materials for \secref{sec:seq-family-pref}: 
Extension of \tableref{tab:main:warmuth_class_sequence_pref}} \label{sec.appendix.seq-models_warmuth}\label{sec.appendix.extension-of-table2}

This section provides the details of preference functions for the Warmuth class.

\begin{table}[h!]\label{tab:warmth_example_families}
\centering
    \begin{subtable}[t]{\textwidth}
        \centering
        \begin{tabular}{l|lllll||l|l|l|l}
        \backslashbox{$\Hypotheses$}{$\Instances$} & $\instance_1$ & 
        $\instance_2$ & 
        $\instance_3$ & 
        $\instance_4$ & 
        $\instance_5$ & 
        $\TeachingSeq_{\textsf{const}}=\TeachingSeq_{\textsf{\glbl}}$ &
        $\TeachingSeq_{\gvs}$ &
        $\TeachingSeq_{\lcl}$ &
        $\TeachingSeq_{\lvs}$ 
        \\ \hline
        $\hypothesis_1$ 
        & 1 & 1 & 0  & 0 & 0 
        & $\paren{\instance_1, \instance_2, \instance_4}$ 
        & $\paren{\instance_1, \instance_2}$
        & $\paren{\instance_1}$ 
        & $\paren{\instance_1}$ 
        \\
        $\hypothesis_2$
        & 0 & 1 & 1 & 0 & 0
        & $\paren{\instance_2, \instance_3, \instance_5}$ 
        & $\paren{\instance_2, \instance_3}$
        & $\paren{\instance_3}$ 
        & $\paren{\instance_2}$ 
        \\
        $\hypothesis_3$
        & 0 & 0 & 1 & 1 & 0
        & $\paren{\instance_1, \instance_3, \instance_4}$ 
        & $\paren{\instance_3, \instance_4}$
        & $\paren{\instance_3, \instance_4}$  
        & $\paren{\instance_3}$ 
        \\
        $\hypothesis_4$
        & 0 & 0 & 0 & 1 & 1
        & $\paren{\instance_2, \instance_4, \instance_5}$ 
        & $\paren{\instance_4, \instance_5}$
        & $\paren{\instance_5, \instance_4}$  
        & $\paren{\instance_4}$ 
        \\
        $\hypothesis_5$
        & 1 & 0 & 0 & 0 & 1
        & $\paren{\instance_1, \instance_3, \instance_5}$ 
        & $\paren{\instance_1, \instance_5}$
        & $\paren{\instance_5}$ 
        & $\paren{\instance_5}$ 
        \\
        $\hypothesis_6$
        & 1 & 1 & 0 & 1 & 0
        & $\paren{\instance_1, \instance_2, \instance_4}$ 
        & $\paren{\instance_2, \instance_4}$
        & $\paren{\instance_4}$ 
        & $\paren{\instance_3}$ 
        \\
        $\hypothesis_7$
        & 0 & 1 & 1 & 0 & 1
        & $\paren{\instance_2, \instance_3, \instance_5}$ 
        & $\paren{\instance_3, \instance_5}$
        & $\paren{\instance_3, \instance_5}$  
        & $\paren{\instance_4}$ 
        \\
        $\hypothesis_8$
        & 1 & 0 & 1 & 1 & 0
        & $\paren{\instance_1, \instance_3, \instance_4}$ 
        & $\paren{\instance_1, \instance_4}$
        & $\paren{\instance_4, \instance_3}$  
        & $\paren{\instance_5}$ 
        \\
        $\hypothesis_9$
        & 0 & 1 & 0 & 1 & 1
        & $\paren{\instance_2, \instance_4, \instance_5}$ 
        & $\paren{\instance_2, \instance_5}$
        & $\paren{\instance_4, \instance_5}$  
        & $\paren{\instance_1}$ 
        \\
        $\hypothesis_{10}$
        & 1 & 0 & 1 & 0 & 1 
        & $\paren{\instance_1, \instance_3, \instance_5}$ 
        & $\paren{\instance_1, \instance_3}$
        & $\paren{\instance_5, \instance_3}$ 
        & $\paren{\instance_2}$ 
        \end{tabular}
        \caption{The Warmuth hypothesis class and the corresponding teaching sequences (denoted by $\TeachingSeq$).}\label{tab:app:warmth_example_sequences}
    \end{subtable}
    \begin{subtable}[t]{.3\textwidth}
        \centering
        \begin{tabular}{c|c}
        $\hypothesis'$ & $\forall \hypothesis' \in \hypotheses$  \\\hline
        $\sigma_{\textsf{const}}(\hypothesis', \cdot, \cdot)$ & \multirow{2}{*}{0}\\
        $\sigma_{\glbl}(\hypothesis', \cdot, \cdot)$ \\
    \end{tabular}
    \caption{$\sigma_{\textsf{const}}$ and $\sigma_{\glbl}$} \label{tab:app:pref_const_global}
    \end{subtable}
    \quad
    \begin{subtable}[t]{.6\textwidth}
        \centering
        
        \begin{tabular}{c|p{.1cm}p{.1cm}p{.1cm}p{.1cm}p{.1cm}p{.1cm}p{.1cm}p{.1cm}p{.1cm}p{.1cm}}
        {$\hypothesis\backslash \hypothesis'$} & 
        $\hypothesis_1$ & $\hypothesis_2$ & $\hypothesis_3$ & $\hypothesis_4$ & $\hypothesis_5$ & $\hypothesis_6$ & $\hypothesis_7$ & $\hypothesis_8$ & $\hypothesis_9$ & $\hypothesis_{10}$ \\\hline
        $\sigma_{\lcl}(\hypothesis'; \cdot, \hypothesis=\hypothesis_1)$ & 
        0 & 2 & 4 & 4 & 2 & 1 & 3 & 3 & 3 & 3 \\
        {\footnotesize $\dots$}&  
    \end{tabular}
    \caption{$\sigma_{\lcl}$ 
    representing the Hamming distance between $\hypothesis'$ and $\hypothesis$. } \label{tab:app:pref_local}
    \end{subtable}
        \begin{subtable}[t]{\textwidth}
    \scalebox{0.8}{
    \begin{tabular}{c|cccccccccc}
    $\hypothesis'$ & $\hypothesis_1$ & $\hypothesis_2$ & $\hypothesis_3$ & $\hypothesis_4$ & $\hypothesis_5$ & $\hypothesis_6$ & $\hypothesis_7$ & $\hypothesis_8$ &  $\hypothesis_9$ & $\hypothesis_{10}$\\\hline
         \multirow{2}{*}{$\hypotheses$} & 
         $\{\hypothesis_1, \hypothesis_6\}$ & 
         $\{\hypothesis_2, \hypothesis_7\}$ &
         $\{\hypothesis_3, \hypothesis_8\}$ & 
         $\{\hypothesis_4, \hypothesis_9\}$ & 
         $\{\hypothesis_5, \hypothesis_{10}\}$ &
         $\{\hypothesis_6, \hypothesis_9\}$ & 
         $\{\hypothesis_7, \hypothesis_{10}\}$ & 
         $\{\hypothesis_8, \hypothesis_6\}$ & 
         $\{\hypothesis_9, \hypothesis_7\}$ & 
         $\{\hypothesis_{10}, \hypothesis_8\}$ \\
        & 
         $\{\hypothesis_1\}$ & 
         $\{\hypothesis_2\}$ &
         $\{\hypothesis_3\}$ & 
         $\{\hypothesis_4\}$ & 
         $\{\hypothesis_5\}$ &
         $\{\hypothesis_6\}$ & 
         $\{\hypothesis_7\}$ & 
         $\{\hypothesis_8\}$ & 
         $\{\hypothesis_9\}$ & 
         $\{\hypothesis_{10}\}$ \\\hline
         $\sigma_\gvs$ & 
         0 & 0 & 0 & 0 & 0 & 0 & 0 & 0 & 0 & 0 \\
    \end{tabular}}
    \caption{$\sigma_\gvs(\hypothesis'; \hypotheses, \cdot)$}\label{tab:app:pref_gvs}
    \end{subtable}
     \begin{subtable}[t]{\textwidth}
        \centering
        \scalebox{0.85}{
        \begin{tabular}{c|cc|cc|cc|cc|cc}
        $\hypothesis'$ & 
         \multicolumn{2}{c|}{$\hypothesis_1$} & 
         \multicolumn{2}{c|}{$\hypothesis_2$} & 
         \multicolumn{2}{c|}{$\hypothesis_3$} & 
         \multicolumn{2}{c|}{$\hypothesis_4$} & 
         \multicolumn{2}{c}{$\hypothesis_5$}
         \\\hline
        \multirow{2}{*}{$\hypotheses$} & \multicolumn{2}{c|}{$\{\hypothesis_1\} \cup$} & \multicolumn{2}{c|}{$\{\hypothesis_2\} \cup$} &
         \multicolumn{2}{c|}{$\{\hypothesis_3\} \cup$} &
         \multicolumn{2}{c|}{$\{\hypothesis_4\} \cup$} &
         \multicolumn{2}{c}{$\{\hypothesis_5\} \cup$}
         \\
        &\multicolumn{2}{c|}{$\{\hypothesis_5, \hypothesis_6,\hypothesis_8,\hypothesis_{10}\}^*$}
        &\multicolumn{2}{c|}{$\{\hypothesis_1, \hypothesis_7,\hypothesis_6,\hypothesis_{9}\}^*$}
        &\multicolumn{2}{c|}{$\{\hypothesis_2, \hypothesis_7,\hypothesis_8,\hypothesis_{10}\}^*$}
        &\multicolumn{2}{c|}{$\{\hypothesis_3, \hypothesis_6,\hypothesis_8,\hypothesis_{9}\}^*$}
        &\multicolumn{2}{c}{$\{\hypothesis_4, \hypothesis_7,\hypothesis_9,\hypothesis_{10}\}^*$}
         \\\hline
         $\hypothesis$ & 
         \multicolumn{2}{c|}{$\hypothesis_1$} & 
         $\hypothesis_1$ & $\hypothesis_2$ &
         $\hypothesis_1$ & $\hypothesis_3$ &
         $\hypothesis_1$ & $\hypothesis_4$ &
         $\hypothesis_1$ & $\hypothesis_5$
         \\\hline
        $\sigma_\lvs$ & 
        \multicolumn{2}{c|}{0} & 
        0 & 0 & 0 & 0 & 0 & 0 & 0 & 0\\
        \end{tabular}
        }\\
        $\vdots$\\
        \scalebox{0.85}{
        \begin{tabular}{c|cc|cc|cc|cc|cc}
         $\hypothesis'$ & 
         \multicolumn{2}{c|}{$\hypothesis_6$} & 
         \multicolumn{2}{c|}{$\hypothesis_7$} & 
         \multicolumn{2}{c|}{$\hypothesis_8$} & 
         \multicolumn{2}{c|}{$\hypothesis_9$} & 
         \multicolumn{2}{c}{$\hypothesis_{10}$}
         \\\hline
                 \multirow{2}{*}{$\hypotheses$} & \multicolumn{2}{c|}{$\{\hypothesis_6\} \cup$} & \multicolumn{2}{c|}{$\{\hypothesis_7\} \cup$} &
         \multicolumn{2}{c|}{$\{\hypothesis_8\} \cup$} &
         \multicolumn{2}{c|}{$\{\hypothesis_9\} \cup$} &
         \multicolumn{2}{c}{$\{\hypothesis_{10}\} \cup$}
         \\ 
          &
         \multicolumn{2}{c|}{$\{\hypothesis_1, \hypothesis_4,\hypothesis_5,\hypothesis_{9}\}^*$}
        &\multicolumn{2}{c|}{$\{\hypothesis_1, \hypothesis_2,\hypothesis_5,\hypothesis_{10}\}^*$}
        &\multicolumn{2}{c|}{$\{\hypothesis_1, \hypothesis_2,\hypothesis_3,\hypothesis_{6}\}^*$}
        &\multicolumn{2}{c|}{$\{\hypothesis_2, \hypothesis_3,\hypothesis_4,\hypothesis_{7}\}^*$}
        &\multicolumn{2}{c}{$\{\hypothesis_3, \hypothesis_4,\hypothesis_5,\hypothesis_{8}\}^*$}
         \\\hline
         $\hypothesis$ & 
         $\hypothesis_1$ & $\hypothesis_6$ &
         $\hypothesis_1$ & $\hypothesis_7$ &
         $\hypothesis_1$ & $\hypothesis_8$ &
         $\hypothesis_1$ & $\hypothesis_9$ &
         $\hypothesis_1$ & $\hypothesis_{10}$ \\\hline
        $\sigma_\lvs$ & 
        0 & 0 & 
        0 & 0 & 0 & 0 & 0 & 0 & 0 & 0\\
    \end{tabular}}
    \caption{$\sigma_\lvs(\hypothesis'; \hypotheses, \hypothesis)$. 
    Here, $\{\cdot \}^*$ denotes all subsets.}\label{tab:app:pref_lvs}
    \end{subtable}
\caption{Teaching sequences with different preference functions for the Warmuth hypothesis class
\cite{doliwa2014recursive}}\label{tab:warmuth_class_sequence_pref}
\end{table}

%% file: 7.3_appendix_seq-models_globalVS-local.tex
\section{Supplementary Materials for \secref{sec:seq-family-pref}: Proof for \thmref{thm:local-eq-GVS}
} \label{sec.appendix.seq-models_globalVS-local}

We divide the proof into three parts. The first part shows that the interactions of the two families is $\SigmaGlobal$. In part 2 and part 3 of the proof, we show that there exist examples of hypothesis classes, such that $\SigmaLocalTD_{\Instances,\Hypotheses,\hinit} > \SigmaGvsTD_{\Instances,\Hypotheses,\hinit}$, or $\SigmaLocalTD_{\Instances,\Hypotheses,\hinit} < \SigmaGvsTD_{\Instances,\Hypotheses,\hinit}$.



\subsection{Part 1}
In this subsection, we provide the full proof for part 1 of \thmref{thm:local-eq-GVS}, i.e., $\SigmaGvs\cap \SigmaLocal=\SigmaGlobal$. 

Intuitively, observe that the input domains between $\sigma_\lcl \in \SigmaGlobal$ and $\sigma_\gvs \in \SigmaGvs$ overlaps at the domain of the first argument, which is the one taken by $\sigma_\glbl$. Therefore, $\forall \sigma \in \SigmaGlobal, \sigma \in \SigmaGvs \cap \SigmaLocal$. We formalize such idea in the proof below. 

\begin{proof}
Assume $\sigma \in \SigmaLocal \cap \SigmaGvs$. Then, by the definitions of $\SigmaLocal$ and $\SigmaGvs$, we get 
\begin{enumerate}[(i)]\denselist
    \item $\exists g^1$, s.t. $\forall \hypothesis, \hypothesis' \in \Hypotheses: \sigma( \hypothesis' ; \cdot, \hypothesis)  = g^1(\hypothesis', \hypothesis)$, and 
    \item $\exists g^2$, s.t. $\forall \hypothesis' \in \Hypotheses, \hypotheses \subseteq \Hypotheses: \sigma( \hypothesis' ; \hypotheses, \cdot)  = g^2(\hypothesis', \hypotheses)$ 
\end{enumerate}
Now consider $\hypothesis', \hypothesis^1, \hypothesis^2 \in \Hypotheses$, and $\hypotheses^1, \hypotheses^2 \subseteq \Hypotheses$. According to (i), $\sigma(\hypothesis'; \hypotheses^1, \hypothesis^1) = \sigma(\hypothesis'; \hypotheses^2, \hypothesis^1)$. Also, according to (ii) $\sigma(\hypothesis'; \hypotheses^2, \hypothesis^1) = \sigma(\hypothesis'; \hypotheses^2, \hypothesis^2)$. This indicates that, $\forall \hypothesis', \hypothesis^1, \hypothesis^2 \in \Hypotheses; \hypotheses^1, \hypotheses^2 \subseteq \Hypotheses: \sigma(\hypothesis'; \hypotheses^1, \hypothesis^1) = \sigma(\hypothesis'; \hypotheses^2, \hypothesis^2)$. In other words, there exist $g^3: \Hypotheses \to \reals$, such that $\forall \hypothesis' \in \Hypotheses: \sigma(\hypothesis';\cdot, \cdot) = g^3(\hypothesis')$. Thus, $\sigma \in \SigmaGlobal$.
\end{proof}
   

\subsection{Part 2}

\paragraph{Part 2.} Next, we show that there exists $(\Hypotheses, \Instances)$, such that $\forall \hinit\in \Hypotheses$, $\SigmaLocalTD_{\Instances,\Hypotheses,\hinit} > \SigmaGvsTD_{\Instances,\Hypotheses,\hinit}$. To prove this statement, we first establish the following lemma.

\begin{lemma}\label{lem:local-rtd-1}
For any $\Hypotheses$, $\Instances$, and $\hinit\in \Hypotheses$, if $\SigmaLocalTD_{\Instances,\Hypotheses,\hinit} = 1$, then $\SigmaGlobalTD_{\Instances,\Hypotheses,\hinit} = 1$.
\end{lemma}
\begin{proof}[Proof of \lemref{lem:local-rtd-1}]
If $\SigmaLocalTD_{\Instances,\Hypotheses,\hinit} = 1$, there should be some $\sigma_{\lcl} \in \SigmaLocal$, such that $\TD_{\Instances,\Hypotheses,\hinit}(\sigma_\lcl) = 1$. Now consider $\sigma_{\glbl}$ such that $\forall \hypothesis', \sigma_\glbl(\hypothesis';\cdot,\cdot) = \sigma_{\lcl}(\hypothesis';\cdot,\hinit)$. If $\Teacher_{\sigma_{\lcl}}$ is the best teacher for $\sigma_{\lcl}$, then $\forall \hypothesis \in \Hypotheses: |T_{\sigma_{\lcl}}(\hypothesis)| = 1$, this indicates that $\hypothesis = \argmin_{\hypothesis' \in \Hypotheses(T_{\sigma_{\lcl}}(\hypothesis))} \sigma_{\lcl}(\hypothesis'; \cdot,\hinit)$ and $|\argmin_{\hypothesis' \in \Hypotheses(T_{\sigma_{\lcl}}(\hypothesis))} \sigma_{\lcl}(\hypothesis';\cdot, \hypothesis_0)| = 1$. Subsequently, $\hypothesis = \argmin_{\hypothesis' \in \Hypotheses(T_{\sigma_{\lcl}}(\hypothesis))} \sigma_{\glbl}(\hypothesis';\cdot,\cdot)$ and $|\argmin_{\hypothesis' \in \Hypotheses(T_{\sigma_{local}})} \sigma_{\glbl}(\hypothesis',\cdot,\cdot)| = 1$. In other words, $\Teacher_{\sigma_{\lcl}}$ is also a teacher for $\sigma_{\lcl}$. This indicates that, $\RTD(\Hypotheses) = \SigmaGlobalTD_{\Instances,\Hypotheses,\hinit} = \TD_{\Instances,\Hypotheses,\hinit}(\sigma_\glbl) = 1$.
\end{proof}

Now we are ready to provide the proof for Part 2.
\begin{proof}[Proof of Part 2 of \thmref{thm:local-eq-GVS}]
We identify $\Hypotheses$, $\Instances$, $\hinit$, where $\SigmaGvsTD_{\Instances,\Hypotheses,\hinit} = 1$ and $\SigmaGlobalTD_{\Instances,\Hypotheses,\hinit} = \RTD = 2$. 
\tableref{tab:batch_model_example_h2} illustrates such an example.
In the example, since $\RTD = 2$,
then by \lemref{lem:local-rtd-1}, it must hold that $\SigmaLocalTD_{\Instances,\Hypotheses,\hinit} > 1 = \SigmaGvsTD_{\Instances,\Hypotheses,\hinit}=\NCTD$.
\end{proof}

\subsection{Part 3}
Here, we show that there exists a problem instance $(\Hypotheses, \Instances)$, such that $\forall \hinit\in \Hypotheses$, $\SigmaLocalTD_{\Instances,\Hypotheses,\hinit} < \SigmaGvsTD_{\Instances,\Hypotheses,\hinit}$.
Consider the hypothesis class which consists of the powerset $\Hypotheses=\{0,1\}^k$. First, as proven in Lemma~\ref{lem:nctd-powerset-kby2} below, we show that $\forall \hinit\in \Hypotheses$, $\SigmaGvsTD_{\Instances,\Hypotheses,\hinit} = \NCTD \geq \ceil{k/2}$. 

 

\begin{lemma}[Based on Theorem 23 of \cite{pmlr-v98-kirkpatrick19a}]\label{lem:nctd-powerset-kby2}
Consider the hypothesis class which consists of the powerset $\Hypotheses=\{0,1\}^k$. Then, $\NCTD \geq \ceil{k/2}$.
\end{lemma}
\begin{proof}
First we make the following observation: If $T$ is a non clashing teacher and $\hypothesis, \hypothesis' \in \Hypotheses$ where $\hypothesis = \hypothesis' \triangle \instance$ (i.e., these two hypotheses only differ in their label on one instance), it must be the case that $(\instance, \hypothesis(\instance)) \in T(\hypothesis)$, or $(\instance, \hypothesis'(\instance)) \in T(\hypothesis')$. This holds by nothing that since $\hypothesis$, and $\hypothesis'$ are only different on $\instance$, if $\instance$ is absent in their teaching sequences, this would lead to violation of the non-clashing property of the teacher.


Next we apply this observation on the powerset $k$ hypotheses class where $\Hypotheses$ consists of all hypotheses which have length $k$. This indicates that for every $\hypothesis \in \Hypotheses$, and $0 \leq j \leq (k-1)$ all k variants $\hypothesis \triangle \instance_j \in \Hypotheses$. For all $0 \leq j \leq (k-1)$ by using the above observation, for a pair $\hypothesis$ and $\hypothesis \triangle \instance_j$, we drive $\sum_{i = 0}^{2^k-1} |T(\hypothesis_i)| \geq \frac{k \cdot 2^k}{2}$. By applying the pigeon-hole principle, this indicates that there exist an $\hypothesis \in \Hypotheses$, where $|T(\hypothesis)| \geq \frac{k}{2}$. In other words $\NCTD(\Hypotheses) \geq \ceil{\frac{k}{2}}$.
\end{proof}


Fix $k=7$, we get $\SigmaGvsTD_{\Instances,\Hypotheses,\hinit} = \NCTD(\Hypotheses) \geq 4$. On the other hand, we construct a preference function $\sigma\in \SigmaLocal$, where $\SigmaLocalTD_{\Instances,\Hypotheses,\hinit} \leq \TD_{\Instances,\Hypotheses,\hinit}(\pref) = 3$ for $k=7$. 

The example is detailed in Figure~\ref{fig:Local_is_less_than_NCTD_7}.  Intuitively, for any $\hinit\in \Hypotheses$, we construct a tree of hypotheses with branching factor $7$ at the top level, branching factor of $6$ at the next level, and so on. Here, each branch corresponds to one teaching example, and each path from $\hinit$ to $\hypothesis\in\Hypotheses$ corresponds to a teaching sequence $\Teacher_\lcl(h)$. We need a tree of depth at most $3$ to include all the $2^7=128$ hypotheses to be taught as nodes in the tree. This gives us a constructive procedure of $\sigma$ functions achieving $\TD_{\Instances,\Hypotheses,\hinit}(\pref) = 3 < \SigmaGvsTD_{\Instances,\Hypotheses,\hinit}$, which completes the proof.


 
\begin{figure}[h!]
    \centering
    \begin{subfigure}[b]{\textwidth}
        \centering
        \includegraphics[width = 0.9\linewidth]{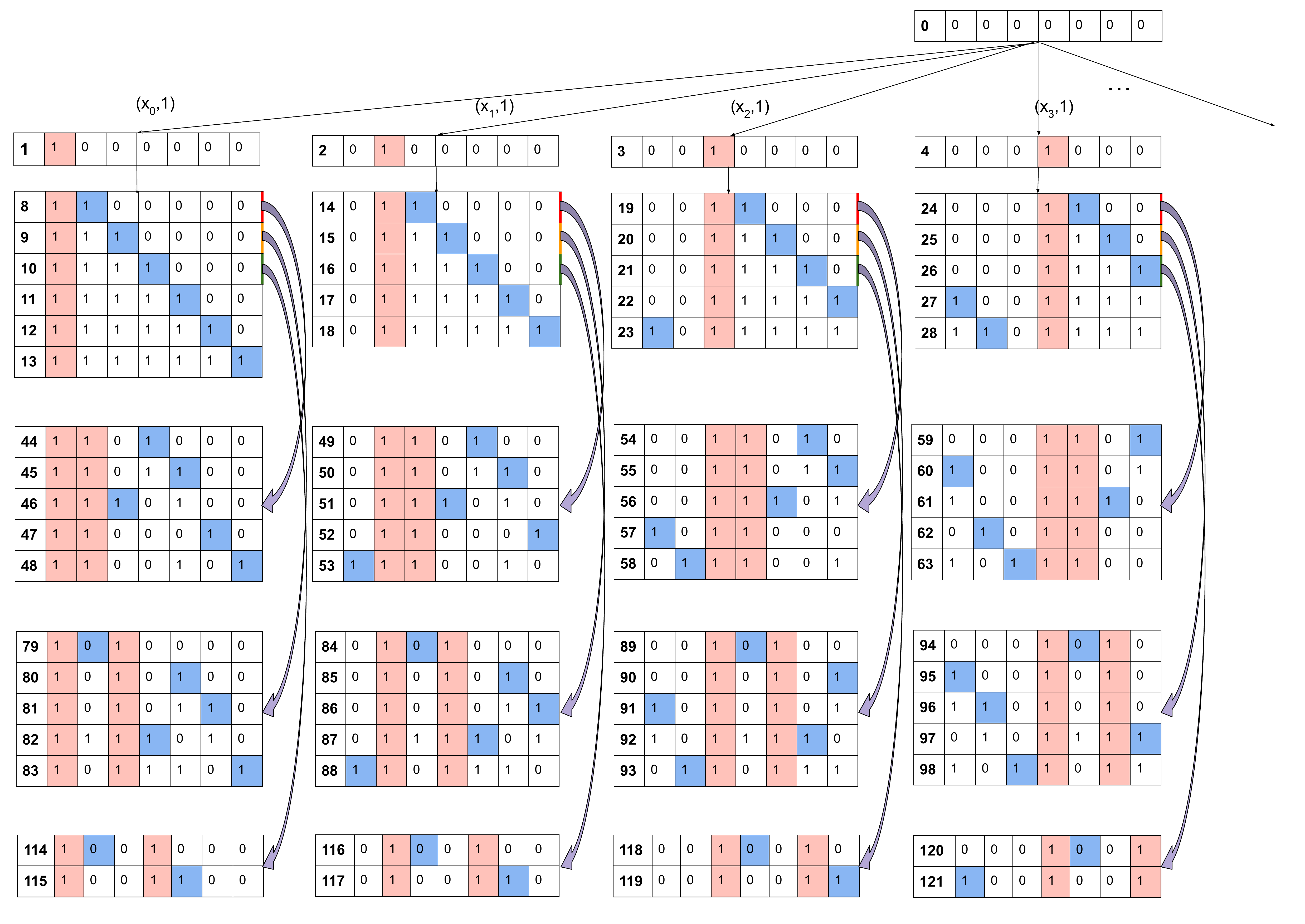}
        \caption{This figure is representing the teaching sequence for first four for direct children of $\hinit$ (top four most preferred hypothesis of $\hinit$ after $\hinit$) and all of their children.}
    \end{subfigure}
    \centering
    \begin{subfigure}[b]{\textwidth}
        \centering
        \includegraphics[width = 0.9\linewidth]{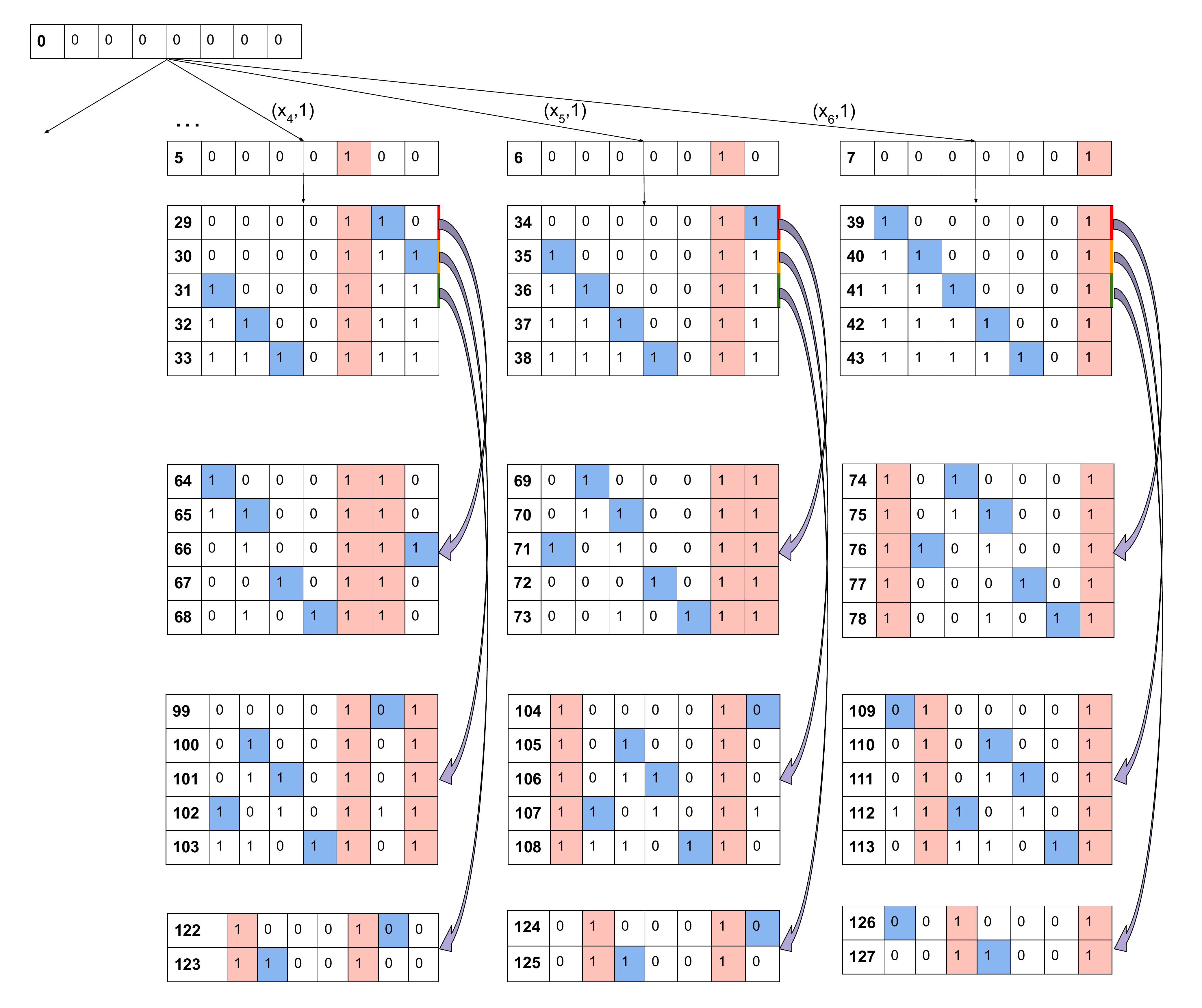}
        \caption{This figure is representing the teaching sequence for next three direct children of $\hinit$ (next three most preferred hypothesis of $\hinit$) and all of their children.}
    \end{subfigure}
    \caption{Details of teaching sequences, for a preference function $\sigma\in \SigmaLocal$, where $\TD_{\Instances,\Hypotheses,\hinit}(\pref) = 3$ for powerset $k=7$ class. For any hypothesis the cell with blue color is representing last teaching example in teaching sequence, and the cells with red color are representing rest of teaching sequence. Also see Table~\ref{tab:Local_is_less_than_NCTD_7} that lists down details for all the hypotheses in the left branch of the tree.}
    \label{fig:Local_is_less_than_NCTD_7}
\end{figure}

\begin{table}[h!]
\centering
\scalebox{0.95}{
\begin{tabular}{ c | c  ||  c | c }

$\hypothesis $ & $\Instances$ & \textbf{Preference Function $\sigma(. ; \hypothesis)$} & \textbf{Teaching Sequence}\\
\hline
\underline{$\hypothesis_0$} & 0 0 0 0 0 0 0 &
$\hypothesis_0 > \hypothesis_1 > \hypothesis_2 > \hypothesis_3 > \hypothesis_4 >$\newline$ \hypothesis_5  > \hypothesis_6 > \hypothesis_7 >$ others &
$\paren{(\instance_0, 0)}$
\\
\hline
$\hypothesis_1$ & 1 0 0 0 0 0 0 &
$\hypothesis_1 > \hypothesis_8 > \hypothesis_9 > \hypothesis_{10} > \hypothesis_{11} > \hypothesis_{12} > \hypothesis_{13} $ others &
$\paren{(\instance_0, 1)}$
\\
\hline
$\hypothesis_8$ & 1 1 0 0 0 0 0 &
$\hypothesis_{8} > \hypothesis_{44} > \hypothesis_{45} > \hypothesis_{46} > \hypothesis_{47} > \hypothesis_{48} >$ others &
$\paren{(\instance_0, 1), (\instance_1, 1)}$
\\
$\hypothesis_{9}$ & 1 1 1 0 0 0 0 &
$\hypothesis_{9} > \hypothesis_{79} > \hypothesis_{80} > \hypothesis_{81} > \hypothesis_{82} > \hypothesis_{83} >$ others &
$\paren{(\instance_0, 1), (\instance_2, 1)}$
\\
$\hypothesis_{10}$ & 1 1 1 1 0 0 0 &
$\hypothesis_{10} > \hypothesis_{114} > \hypothesis_{115} >$ others &
$\paren{(\instance_0, 1), (\instance_3, 1)}$
\\
$\hypothesis_{11}$ & 1 1 1 1 1 0 0 &
$\hypothesis_{11} >$ others &
$\paren{(\instance_0, 1), (\instance_4, 1)}$
\\
$\hypothesis_{12}$ & 1 1 1 1 1 1 0 &
$\hypothesis_{12} >$ others &
$\paren{(\instance_0, 1), (\instance_5, 1)}$
\\
$\hypothesis_{13}$ & 1 1 1 1 1 1 1 &
$\hypothesis_{13} >$ others &
$\paren{(\instance_0, 1), (\instance_6, 1)}$
\\
\hline
$\hypothesis_{44}$ & 1 1 0 1 0 0 0 &
$\hypothesis_{44} >$ others &
$\paren{(\instance_0, 1), (\instance_1, 1), (\instance_3, 1)}$
\\
$\hypothesis_{45}$ & 1 1 0 1 1 0 0 &
$\hypothesis_{45}>$ others &
$\paren{(\instance_0, 1), (\instance_1, 1), (\instance_4, 1)}$
\\
$\hypothesis_{46}$ & 1 1 1 0 1 0 0 &
$\hypothesis_{46} >$ others &
$\paren{(\instance_0, 1), (\instance_1, 1), (\instance_2, 1)}$
\\
$\hypothesis_{47}$ & 1 1 0 0 0 1 0 &
$\hypothesis_{47} >$ others &
$\paren{(\instance_0, 1), (\instance_1, 1), (\instance_5, 1)}$
\\
$\hypothesis_{48}$ & 1 1 0 0 1 0 1 &
$\hypothesis_{48} >$ others &
$\paren{(\instance_0, 1), (\instance_1, 1), (\instance_6, 1)}$
\\
\hline
$\hypothesis_{79}$ & 1 0 1 0 0 0 0 &
$\hypothesis_{79} >$ others &
$\paren{(\instance_0, 1), (\instance_2, 1), (\instance_1, 0)}$
\\
$\hypothesis_{80}$ & 1 0 1 0 1 0 0 &
$\hypothesis_{80}>$ others &
$\paren{(\instance_0, 1), (\instance_2, 1), (\instance_4, 1)}$
\\
$\hypothesis_{81}$ & 1 0 1 0 1 1 0 &
$\hypothesis_{81} >$ others &
$\paren{(\instance_0, 1), (\instance_2, 1), (\instance_5, 1)}$
\\
$\hypothesis_{82}$ & 1 1 1 1 0 1 0 &
$\hypothesis_{47} >$ others &
$\paren{(\instance_0, 1), (\instance_2, 1), (\instance_3, 1)}$
\\
$\hypothesis_{83}$ & 1 0 1 1 1 0 1 &
$\hypothesis_{83} >$ others &
$\paren{(\instance_0, 1), (\instance_2, 1), (\instance_6, 1)}$
\\
\hline
$\hypothesis_{114}$ & 1 0 0 1 0 0 0 &
$\hypothesis_{114} >$ others &
$\paren{(\instance_0, 1), (\instance_3, 1), (\instance_1, 0)}$
\\
$\hypothesis_{115}$ & 1 0 0 1 1 0 0 &
$\hypothesis_{115}>$ others &
$\paren{(\instance_0, 1), (\instance_3, 1), (\instance_4, 1)}$
\\

\end{tabular}}
\vspace*{2mm}
\caption{More details about Figure~\ref{fig:Local_is_less_than_NCTD_7}: This table lists down all the hypotheses in the left branch of the tree. For each of these hypotheses, it shows the preference function from the hypothesis, as well as the teaching sequence to teach the hypothesis. Consider $\hypothesis_9$: We have $\sigma(., \hypothesis_9) = \{\hypothesis_{9} > \hypothesis_{79} > \hypothesis_{80} > \hypothesis_{81} > \hypothesis_{82} > \hypothesis_{83} > \textnormal{others}\}$. Also, we have teaching sequence for $\hypothesis_9$ as $\{(\instance_0, 1), (\instance_2, 1)\}$.}
\label{tab:Local_is_less_than_NCTD_7}
\end{table}

%% file: 7.4.1_appendix_seq-models_linearVCD_first-lemma_non-empty.tex
\section{Supplementary Materials for \secref{sec:lvs_linvcd}} \label{sec.appendix.seq-models_linearVCD}
\subsection{Proof of \lemref{lem:main:div_concepts}}\label{sec.appendix.seq-models_linearVCD.lemmaproof}
In this section, we extend the proof sketch of \lemref{lem:main:div_concepts} in the main paper into the full proof.  A useful notion for this proof is the notion of \emph{$\hypotheses$-distinguishable set}: 
\begin{definition}[Based on \cite{doliwa2014recursive}]\label{def:hdistinguishable}
A set of instances $\instances \subseteq \Instances$ is $\hypotheses$-distinguishable, if $|\hypotheses_{|\instances}| = |\hypotheses|$.
\end{definition}
For completeness, we also incorporate part of the proof sketch from \secref{sec:lvs_linvcd} into the extended proof below.
\begin{proof}[(Extended) Proof of \lemref{lem:main:div_concepts}]
Let us define $\hypotheses_{\instance} = \{\hypothesis \in \hypotheses: {\hypothesis \triangle \instance}_{|\compactDSet{{\hypotheses}}} \in \hypotheses_{|\compactDSet{{\hypotheses}}}\}$. Here, $\hypothesis \triangle \instance$ denotes the hypothesis that only differs from $\hypothesis$ on the label of $\instance$. Fix a reference hypothesis $\hypothesis_\hypotheses$. $\forall \instance_j \in \compactDSet{\hypotheses}$, let $\clabel_j = 1 -  \hypothesis_\hypotheses(\instance_j)$ be the opposite label of $\instance_j$ as provided by $\hypothesis_\hypotheses$.  As highlighted in \lnref{alg:linvcd:partition} of \algref{alg:lemma_linVCD}, we consider the set $\hypotheses^{\clabel_1}_{\instance_1} = \{\hypothesis \in \hypotheses_{\instance_1}: \hypothesis(\instance_1) = \clabel_1\}$ as the first partition. 

\begin{table}[h!]
\begin{subtable}{.32\linewidth}
\scalebox{0.78}{
\begin{tabular}{|l|l|l|l|l|}
\hline
\backslashbox{$\hypotheses$}{$\compactDSet{\hypotheses}$} & $\instance_1$ & $\instance_2$ & ... & $\instance_{m-1}$ \\ \hline
$\hypothesis_0$                     & 0             & \multicolumn{3}{c|}{00 ... 0}              \\ 
$\hypothesis_1$                     & 0             & \multicolumn{3}{c|}{a}              \\ 
$\hypothesis_2$                     & 1             & \multicolumn{3}{c|}{b}              \\ 
$\hypothesis_3$                     & 0             & \multicolumn{3}{c|}{b}              \\ 
$\hypothesis_4$                     & 1             & \multicolumn{3}{c|}{c}              \\ 
$\hypothesis_5$                     & 1             & \multicolumn{3}{c|}{d}              \\ 
$\hypothesis_6$                     & 1             & \multicolumn{3}{c|}{e}              \\ 
$\hypothesis_7$                     & 1             & \multicolumn{3}{c|}{a}              \\ \hline
\end{tabular}
}
\caption{$\hypotheses$}\label{tab:H-1.1}
\end{subtable}
~
\begin{subtable}{.32\linewidth}
\scalebox{0.78}{
\begin{tabular}{|l|l|l|l|l|}
\hline
\backslashbox{$\hypotheses_{\instance_1}$}{$\compactDSet{\hypotheses}$} & $\instance_1$ & $\instance_2$ & ... & $\instance_{m-1}$ \\ \hline
$\hypothesis_1$                     & 0             & \multicolumn{3}{c|}{a}              \\ 
$\hypothesis_2$                     & 1             & \multicolumn{3}{c|}{b}              \\ 
$\hypothesis_3$                     & 0             & \multicolumn{3}{c|}{b}              \\ 
$\hypothesis_7$                     & 1             & \multicolumn{3}{c|}{a}              \\ \hline
\end{tabular}
}
\caption{$\hypotheses_{\instance_1}$}\label{tab:H-1.2}
\end{subtable}
~
\begin{subtable}{.32\linewidth}
\scalebox{0.78}{
\begin{tabular}{|l|l|l|l|l|}
\hline
\backslashbox{$\hypotheses^1$}{$\compactDSet{\hypotheses}$} & $\instance_1$ & $\instance_2$ & ... & $\instance_{m-1}$ \\ \hline
$\hypothesis_2$                     & 1             & \multicolumn{3}{c|}{b}              \\ 
$\hypothesis_7$                     & 1             & \multicolumn{3}{c|}{a}              \\ \hline
\end{tabular}
}
\caption{$\hclass{1} = \hypotheses^{\clabel_1=1}_{\instance_1}$}\label{tab:H-1.3}
\end{subtable}
\caption{Illustrative example for constructing the first partition $\hclass{1} = \hypotheses^{\clabel_1=1}_{\instance_1}$. 
}\label{tab:H-1}
\end{table}
In \tableref{tab:H-1}, we provide an example hypothesis class where we show how to construct the first partition $\hypotheses^{\clabel_1}_{\instance_1}$. \tableref{tab:H-1.1} shows the hypothesis class $\Hypotheses$ (here $a \neq b \neq c \neq d \neq e$) and $\hypothesis_\Hypotheses = \hypothesis_0$. \tableref{tab:H-1.2} shows the resulting set of hypotheses 
$\hypotheses_{\instance_1} = \{\hypothesis \in \hypotheses: {\hypothesis \triangle \instance_1}_{|\compactDSet{{\hypotheses}}} \in \hypotheses_{|\compactDSet{{\hypotheses}}}\}$, and \tableref{tab:H-1.3} shows the first partition $\hypotheses^{\clabel_1=1}_{\instance_1}$.

We denote $\hclass{1} := \hypotheses^{\clabel_1}_{\instance_1}$, and define $\compactDSet{\hclass{1}} \subseteq \compactDSet{\hypotheses} \setminus \{\instance_1\}$ to be any compact-distinguishable set on $\hclass{1}$. 


\paragraph{Lower VCD.} Let $d=\VCD(\hypotheses, \compactDSet{\hypotheses})$. In the following, we prove that $\VCD(\hclass{1}, \compactDSet{\hclass{1}}) \leq d-1$. We consider the following two cases:
\begin{enumerate}
    \item If $d>1$, then
\begin{equation*}
     \VCD(\hclass{1}, \compactDSet{\hclass{1}}) \leq \VCD(\hypotheses^{\clabel_1}_{\instance_1}, \compactDSet{\hypotheses}) = \VCD(\hypotheses_{\instance_1}, \compactDSet{\hypotheses}) - 1 \leq \VCD(\hypotheses, \compactDSet{\hypotheses}) - 1 \leq d - 1
\end{equation*}
Since $\compactDSet{\hclass{1}} \subset \compactDSet{\hypotheses}$, the first inequality is due to the monotonicity of $\VCD$. The equality follows from the fact that, for all $\hypothesis \in \hypotheses^{\clabel_1}_{\instance_1}$, it holds that $\hypothesis(\instance_1) = \clabel_1$ and $ {\hypothesis \triangle \instance_1}_{|\compactDSet{\hypotheses}} \in {\hypotheses_{{\instance_1}}}_{|\compactDSet{\hypotheses}}$. This indicates that, $\instances \subseteq \compactDSet{\hypotheses}$ shatters $\hypotheses^{\clabel_1}_{\instance_1}$, iff $\instances \cup \{\instance_1\}$ shatters $\hypotheses_{\instance_1}$. The second inequality comes from the fact that $\VCD$ is monotonic.
\item If $d=1$ and $|\hypotheses^{\clabel_1}_{\instance_1}| \geq 2$, then\\
similar to the previous case we have the following: $\VCD(\hypotheses_{\instance_1}, \compactDSet{\hypotheses})\leq \VCD(\hypotheses, \compactDSet{\hypotheses}) = 1$ and $\VCD(\hypotheses_{\instance_1}, \compactDSet{\hypotheses}) = \VCD(\hypotheses^{\clabel_1}_{\instance_1}, \compactDSet{\hypotheses}) + 1$. Subsequently, $\VCD(\hclass{1}, \compactDSet{\hclass{1}}) = 0$.
\item If $d=1$ and $|\hypotheses^{\clabel_1}_{\instance_1}| < 2$, then\\
since $|\hypotheses^{\clabel_1}_{\instance_1}| < 2$, by definition, we have $\VCD(\hclass{1}, \compactDSet{\hclass{1}}) = 0$ and hence is less than $d=1$.
\end{enumerate}

That is, the first partition $\hclass{1}, \compactDSet{\hclass{1}}$ has $\VCD(\hclass{1}, \compactDSet{\hclass{1}}) \leq d - 1$, i.e., $\hclass{1}$ satisfies property (i). In addition, it is clear that ${\hclass{1}}_{|\{\instance_1\}} = \{\clabel_1\} = \{1 - \hypothesis_\hypotheses(\instance_1)\}$. Therefore, $\hclass{1}$ also satisfies property (ii).


\paragraph{Non-emptiness of $\hypotheses^1$.} For the sake of contradiction assume that $\hypotheses^1$ is empty. Note that $\compactDSet{\hypotheses}$ is $\hypotheses$-distinguishable. Since $\hypotheses^1$ is empty, this means that there is no pair of hypotheses that differ only on $x_1$. This in turn indicates that $\compactDSet{\hypotheses} \setminus \{\instance_1\}$ is $\hypotheses$-distinguishable. However, $|\compactDSet{\hypotheses} \setminus \{\instance_1\}| < |\compactDSet{\hypotheses}|$ and this is in contradiction to the assumption that $\compactDSet{\hypotheses}$ is compact-distinguishable on $\hypotheses$. 

\paragraph{Continuing to create partitions.} Next, we remove the first partition $\hclass{1}$ from $\hypotheses$, and continue to create the above mentioned partitions on $\hypotheses_{\text{rest}} = \hypotheses \setminus \hclass{1}$ and $\instances_{\text{rest}} = \compactDSet{\hypotheses} \setminus \{\instance_1\}$.  We claim that $\hypotheses_{\text{rest}}, \instances_{\text{rest}}$ exhibit the following properties.
\begin{enumerate}
    \item $\instances_{\text{rest}}$ is $\hypotheses_{\text{rest}}$-distinguishable (see \defref{def:hdistinguishable}).

    For the sake of contradiction, assume that there exists a pair of hypotheses $\hypothesis^1, \hypothesis^2\in  \hypotheses_{\text{rest}}$ such that $\hypothesis^1_{|\instances_{\text{rest}}} = \hypothesis^2_{|\instances_{\text{rest}}}$. However, we know that $\hypothesis^1_{|\compactDSet{\hypotheses}} \neq \hypothesis^2_{|\compactDSet{\hypotheses}}$. Then, these two hypotheses should have been in $\hypotheses_{\instance_1}$ and only one of them could have stayed in $\hypotheses_{\text{rest}}$. Hence, there is no such pair of hypotheses in $\hypotheses_{\text{rest}}$ and this completes the proof of the statement.
    
    \item $\instances_{\text{rest}}$ is also a compact-distinguishable on $\hypotheses_{\text{rest}}$.
    
    We now provide a concrete proof for the above statement. Imagine $\instances \subseteq \instances_{\text{rest}}$ is an $\hypotheses_{\text{rest}}$-distinguishable set. In the following, we prove that $\instances \cup \{\instance_1\}$ is $\hypotheses$-distinguishable.

    For the sake of contradiction assume that, $\instances \cup \{\instance_1\}$ isn't $\hypotheses$-distinguishable. This indicates that there exist two hypotheses $\hypothesis^1 \neq \hypothesis^2 \in \hypotheses$, where they are equal on $\instances \cup \{\instance_1\}$, i.e., $\hypothesis^1_{|\instances \cup \{\instance_1\}} = \hypothesis^2_{|\instances \cup \{\instance_1\}}$; also this implies $\hypothesis^1_{|\instances} = \hypothesis^2_{|\instances}$. Since $\hypotheses = \hypotheses_{\text{rest}} \cup \hypotheses^1$, we consider the following three cases.
    \begin{enumerate}[(i)]
        \item $\hypothesis^1, \hypothesis^2 \in \hypotheses_{\text{rest}}$. Since $\instances$ is $\hypotheses_{\text{rest}}$-distinguishable, it is a contradiction that $\hypothesis^1_{|\instances} = \hypothesis^2_{|\instances}$.
        \item $\hypothesis^1, \hypothesis^2 \in \hypotheses^1$. By the construction of $\hypotheses^1$, there exist $\hat{\hypothesis}^1, \hat{\hypothesis}^2 \in \hypotheses_{\text{rest}}$, such that $\hat{\hypothesis}^1_{|\instances \cup\{\instance_1\}} = {\hypothesis^1 \triangle \instance_1}_{|\instances \cup\{\instance_1\}}$ and $\hat{\hypothesis}^2_{|\instances \cup\{\instance_1\}} = {\hypothesis^2 \triangle \instance_1}_{|\instances \cup\{\instance_1\}}$. Furthermore, since $\hypothesis^1_{|\instances} = \hypothesis^2_{|\instances}$, we must have $\hat{\hypothesis}^1_{|\instances} = \hat{\hypothesis}^2_{|\instances}$, which contradicts the fact that $\instances$ is $\hypotheses_{\text{rest}}$-distinguishable.
        \item $\hypothesis^1 \in \hypotheses^1, \hypothesis^2 \in \hypotheses_{\text{rest}}.$
        By the construction of $\hypotheses^1$, there exist $\hat{\hypothesis}^1 \in \hypotheses_{\text{rest}}$, such that $\hat{\hypothesis}^1_{|\instances \cup\{\instance_1\}} = {\hypothesis^1 \triangle \instance_1}_{|\instances \cup\{\instance_1\}}$.
        Furthermore, since $\hypothesis^1_{|\instances} = \hypothesis^2_{|\instances}$, we must have $\hat{\hypothesis}^1_{|\instances} = \hypothesis^2_{|\instances}$, which contradicts the fact that $\instances$ is $\hypotheses_{\text{rest}}$-distinguishable.

    \end{enumerate}
    Therefore, we conclude that $\instances \cup \{\instance_1\}$ is $\hypotheses$-distinguishable. Recall that $\compactDSet{\hypotheses}$ is compact-distinguishable on $\hypotheses$. This indicates that $\compactDSet{\hypotheses} = \instances \cup \{\instance_1\}$, and subsequently $\instances = \instances_{\text{rest}}$. This indicates that $\instances_{\text{rest}}$ is compact-distinguishable on $\hypotheses_{\text{rest}}$.

    \item If $|\compactDSet{\hypotheses}| > 1$, then $|\hypotheses_{\text{rest}}| > 1$.

    We prove the above statement by contradiction. Assume that $|\hypotheses_{\text{rest}}| = 1$. Since we know that $\hypotheses^1$ is non empty, hence $|\hypotheses_{\text{rest}}| = 1$ implies that $|\hypotheses^1| = 1$.
    Let $\hclass{1} = \{\hypothesis\}$, and $\hypotheses_{\text{rest}} = \{\hypothesis'\}$, then $\hypothesis'_{|\compactDSet{\hypotheses}} = {\hypothesis \triangle \instance_1}_{|\compactDSet{\hypotheses}}$. Since we know that $\hypotheses = \hclass{1} \cup \hypotheses_{\text{rest}}$ , subsequently $\{\instance_1\}$ is compact-distinguishable on $\hypotheses$, which is in contradiction to the assumption that $\compactDSet{\hypotheses}$ is compact-distinguishable.
\end{enumerate}

\paragraph{Case of $|\instances_{\text{rest}}| > 1$.}  Therefore, we can repeat the above procedure (\lnref{alg:ln:lemma_linvcd_start}-- \lnref{alg:ln:lemma_linvcd_end}, \algref{alg:linvcd}) to create the subsequent partitions. This process continues until the size of $\instances_{\text{rest}}$ reduces to $1$, i.e. $\instances_{\text{rest}} = \{\instance_{m-1}\}$. Until then, we obtain partitions $\{\hclass{1}, ..., \hclass{m-2}\}$. By construction, $\hclass{j}$ satisfy properties (i) and (ii) for all $j \in [m-2]$.

Note that each step $\instances_{\text{rest}}$ is compact $\hypotheses_{\text{rest}}$-distinguishable set. This implies that we have never lost a hypothesis in this process, i.e., all hypotheses in $\hypotheses$ were either in one of $\hypotheses_j$'s or in $\hypotheses_{\text{rest}}$.

\paragraph{Case of $|\instances_{\text{rest}}| = 1$.}  It remains to show that the last two partitions $\hclass{m-1}$ and $\hclass{m}$ also satisfy properties (i) and (ii); additionally we need to satisfy property (iii). Since $\instances_{\text{rest}} = \{\instance_{m-1}\}$, and $|\hypotheses_{\text{rest}}| > 1$ before we start iteration $m-1$, there must exist exactly two hypotheses in $\hypotheses_{\text{rest}}$. Therefore $|\hclass{m-1}|, |\hclass{m}| = 1$, and $\hclass{m-1}_{|\{\instance_{m-1}\}} = \{\{\clabel_{m-1}\}\}$. This implies that $\VCD(\hclass{m-1}, \compactDSet{\hclass{m-1}}) = \VCD(\hclass{m}, \compactDSet{\hclass{m}}) = 0 \leq d-1$. Furthermore, notice that for every $j \in [m-1], \hypothesis \in \hclass{j}, \hypothesis_\hypotheses(\instance_j) \neq \hypothesis(\instance_j)$. This  indicates $\hypothesis_\hypotheses \in \hypotheses_m$. Since $|\hypotheses_m| = 1$, we get $\hypotheses_m = \{\hypothesis_\hypotheses\}$ which completes the proof.
\end{proof}

%% file: 7.4.3_appendix_seq-models_linearVCD_theorem.tex
\subsection{Supplementary Materials for the Proof of Theorem~\ref{thm:main:seq-models_vs_VCD}}\label{sec.appendix.seq-models_linearVCD.theoremproof}

Our proof of \thmref{thm:main:seq-models_vs_VCD} in the main paper relies on the fact that every teaching
example $(\instance_j, \clabel_j)$, where $\instance_j\in \compactDSet{\hypotheses}$ and $\clabel_j=1-\hypothesis(x_j)$ for some fixed $h$, corresponds to a unique version space $V^j$. The proof depends on the following lemma.

\begin{lemma}\label{lem:app:vs_dependent_TD}
 
 Fix $\hypotheses \subseteq \Hypotheses$, and let $\compactDSet{\hypotheses} \subseteq \Instances$ be a compact-distinguishable set on $\hypotheses$. For any $\instance, \instance' \in \compactDSet{\hypotheses}$ and $\clabel, \clabel' \in \{0,1\}$ such that $(\instance, \clabel) \neq (\instance', \clabel')$, the resulting version spaces $\{\hypothesis \in \hypotheses: \hypothesis(\instance) = \clabel\}$ and $\{\hypothesis \in \hypotheses: \hypothesis(\instance') = \clabel'\}$ are different.
 
\end{lemma}


\begin{proof}[Proof of \lemref{lem:app:vs_dependent_TD}]
Denote $A = \{\hypothesis \in \hypotheses: \hypothesis(\instance) = \clabel\}$ and $B = \{\hypothesis \in \hypotheses: \hypothesis(\instance') = \clabel'\}$. We consider the following two cases: (1) $\clabel = \clabel'$, and (2) $\clabel \neq \clabel'$. For the first case where $\clabel = \clabel'$, if  $A  = B$, this would violate the first part of property (i) of \lemref{lem:main:div_concepts}, (i.e., there do not exist distinct $\instance, \instance'$ s.t. $\forall \hypothesis \in \hypotheses: \hypothesis(\instance) = \hypothesis(\instance')$. For the second case, if $A  = B$, this would violate the second part of property (i). Hence it completes the proof. 
\end{proof}